\documentclass[journal]{IEEEtran}
\IEEEoverridecommandlockouts
\usepackage{cite}
\usepackage{amsmath,amssymb,amsfonts}
\usepackage{algorithmic}
\usepackage{graphicx}
\graphicspath{ {./figures/} }
\usepackage{textcomp}
\usepackage{xcolor}
\usepackage{bbm}
\usepackage{comment}
\usepackage{paralist}
\usepackage[caption=false]{subfig}
\usepackage{scalerel}    

\usepackage{amsthm}
\theoremstyle{plain}

\usepackage[linesnumbered,ruled]{algorithm2e}
\theoremstyle{remark}
\newtheorem{remark}{Remark}
\theoremstyle{definition}
\newtheorem{definition}{Definition}
\theoremstyle{plain}
\newtheorem{theorem}{Theorem}
\newtheorem{corollary}{Corollary}
\newtheorem{prop}{Proposition}
\newtheorem{problem}{Problem}
\newtheorem{lemma}{Lemma}

\usepackage{dirtytalk}
\usepackage[hidelinks]{hyperref}
\let\oldnl\nl
\newcommand{\nonl}{\renewcommand{\nl}{\let\nl\oldnl}}

\DeclareMathOperator*{\argmax}{argmax} 
\DeclareMathOperator*{\argmin}{argmin} 
   
\begin{document}
\title{Robust Multiple-Path Orienteering Problem:\\ Securing Against Adversarial Attacks
\thanks{This work is supported by the National Science Foundation under Grant No. 1943368, and the Office of Naval Research under Grant No. N000141812829.}
\thanks{
\textsuperscript{$\dagger$}Department of Electrical and Computer Engineering, University of Maryland, College Park, MD 20742 USA email:gyshi@terpmail.umd.edu.}
\thanks{
\textsuperscript{$\ddagger$}GRASP Laboratory,
University of Pennsylvania, Philadelphia, PA, USA e-mail:
lfzhou@seas.upenn.edu. The author was with the Department of Electrical and Computer Engineering, Virginia Tech, Blacksburg, VA, USA when part of the work was completed.}
\thanks{\textsuperscript{$\dagger\dagger$}Department of Computer Science, University of Maryland, College Park, MD 20742 USA email: tokekar@umd.edu}
}

\author{\IEEEauthorblockN{Guangyao Shi\textsuperscript{$\dagger$}} \and \IEEEauthorblockN{Lifeng Zhou\textsuperscript{$\ddagger$}} \and 
\IEEEauthorblockN{Pratap Tokekar\textsuperscript{$\dagger\dagger$}}\\
}
\maketitle

\begin{abstract}
The multiple-path orienteering problem asks for paths for a team of robots that maximize the total reward collected while satisfying budget constraints on the path length. This problem models many multi-robot routing tasks such as exploring unknown environments and information gathering for environmental monitoring. In this paper, we focus on how to make the robot team robust to failures when operating in adversarial environments.  
We introduce the Robust Multiple-path Orienteering Problem (RMOP) where we seek worst-case guarantees against an adversary that is capable of attacking at most $\alpha$ robots. 
We consider two versions of this problem: RMOP offline and RMOP online. In the offline version, there is no communication or replanning when robots execute their plans and our main contribution is a general approximation scheme with a bounded approximation guarantee that depends on $\alpha$ and the approximation factor for single robot orienteering. In particular, we show that the algorithm yields a (i) constant-factor approximation when the cost function is modular; (ii) $\log$ factor approximation when the cost function is submodular; and (iii) constant-factor approximation when the cost function is submodular but the robots are allowed to exceed their path budgets by a bounded amount. In the online version, RMOP is modeled as a two-player sequential game and solved adaptively in a receding horizon fashion based on Monte Carlo Tree Search (MCTS). 
In addition to theoretical analysis, we perform simulation studies for ocean monitoring and tunnel information-gathering applications to demonstrate the efficacy of our approach.
\end{abstract}

\begin{IEEEkeywords}
Multi-robot system, Orienteering, robust planning, adversarial attacks.
\end{IEEEkeywords}

\section{Introduction}

The Orienteering Problem (OP) is that of determining a path, whose length is less than a given budget, from a given starting vertex that maximizes the total reward collected along the path~\cite{vansteenwegen2011orienteering}. The reward depends on the vertices visited along the path. The OP\footnote{Unless specified otherwise, OP refers to single robot orienteering.} naturally models informative-path planning: a robot is tasked to gather as much information from the environment as possible within a given time or energy budget. 
\begin{figure}[htbp]
\centerline{\includegraphics[scale=0.7]{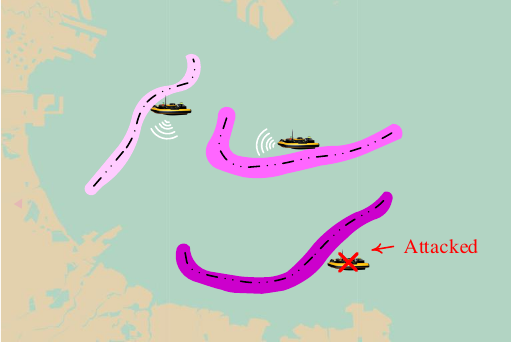}}
\caption{Case study of monitoring a marine environment with aquatic robots. The robots are tasked with finding informative paths to gather data. The darker the color of the path is, the more valuable path since it gathers information from a more important region. We investigate the question of how the robots should plan their paths if we expect some of the robots to fail due to adversarial elements or natural causes\textcolor{red}{.}
}
\label{fig:illustration}
\end{figure}
For example, in \cite{jorgensen2017matroid,thakur2013planning,sadeghi2019minimum}, ocean monitoring, opportunistic surveillance, and 3D reconstruction tasks are formulated as the OP or its variants. In general, the OP is NP-hard but there are constant-factor approximation algorithms for many variants~\cite{blum2003approximation}. This includes the Multiple-path Orienteering Problem (MOP)~\cite{blum2003approximation} where the goal is to design paths for $N$ robots such that the sum of the rewards collected by all the robots is maximized.
In this paper, we introduce the robust variant of OP. Specifically, we introduce the Robust Multiple-Path Orienteering Problem (RMOP) motivated by scenarios where robots operate in adversarial or failure-prone environments. 



Figure~\ref{fig:illustration} shows a motivating scenario where a team of underwater robots is tasked with gathering data in an ocean. However, some robots in the team may fail to complete their paths either due to adversarial attacks~\cite{sless2014multi} or hardware malfunction~\cite{carlson2004follow}. If a robot fails, then the data gathered by it is lost. Our goal is to provide efficient planning and coordination algorithms that are resilient to such failures.

Building robot teams that are robust to adversarial attacks is emerging as an important research area~\cite{prorok2019guest,guerrero2017formations,prorok2019redundant,saulnier2017resilient}. Our approach differs from classical fault-tolerant frameworks~\cite{spong1992robust,ting1994control,crestani2012fault} that focus on making individual robots robust to failures. Instead, we focus on the question of how should the team coordinate their actions to improve redundancy in their plans such that even if some robots fail, the overall performance of the team will not drop significantly. As such, our work is completely different from the work on making individual robots robust. 

In this paper, we focus on the RMOP to make progress towards the aforementioned broader goal. The RMOP seeks plans for a team of $N$ robots that guard against worst-case failures. Of course, in the worst-case, all $N$ robots may fail. To make it more meaningful we study the case where at most a given number $\alpha<N$ robots may fail. What we seek is to understand how the performance of the team will be affected as a function of $\alpha$.
We consider two types of RMOP, {in both of which we seek to find a path consisting of multiple steps for each robot}. In the offline RMOP in which robots cannot communicate with each other or the base station during tasks, our main contribution is an algorithmic scheme that uses a single robot OP solution as a subroutine. Choosing an appropriate subroutine allows us to investigate three variants of the original problem. In the general version, the reward collected by an individual robot is a submodular function of the vertices along the path. Submodularity is the property of diminishing returns~\cite{clark2016submodularity}. Many information gathering measures such as mutual information~\cite{krause2008near} and coverage area~\cite{krause2011submodularity} are known to be submodular. We also study special cases where the reward function is strictly modular (i.e., additive) and where the budget constraint for each robot can be relaxed by a bounded amount. In the online RMOP, we model the problem as a sequential two-player game and propose an {adaptive strategy} based on MCTS, and the problem is solved in a receding horizon fashion {with the history of the observed attack taken into account}.


\subsection{Related work} \label{sec:relwork}
The orienteering problem has been researched extensively by both theoretical computer science and operations research communities. The review by Vansteenwegen~\cite{vansteenwegen2011orienteering} summarizes various algorithms for OP and its variants. We highlight the results most closely related to our work. Blum et al.~\cite{blum2003approximation} presented a polynomial-time 4--approximation for OP when the objective function is modular. This result is then extended to yield a 5--approximation for the MOP assuming all robots start at different vertices. If the reward function is submodular, Chekuri and Pal~\cite{recursivegreedy2005} present a recursive greedy algorithm for a single robot that yields a $O(\log(OPT))$ approximation algorithm, where $OPT$ is the reward collected by the optimal algorithm. The algorithm runs in quasi-polynomial time. 

Singh et al.~\cite{singh2009efficient} showed how to use OP and MOP for active information gathering to learn a spatial model of the environment represented by Gaussian Processes. Their algorithms sequentially find paths for each robot using the single-robot algorithms by Blum et al.~\cite{blum2003approximation} and Chekuri and Pal~\cite{recursivegreedy2005} as subroutines.
Atanasov et al.~\cite{atanasov2015decentralized} recently presented a decentralized version for multi-robot information gathering along similar lines as~\cite{singh2009efficient}. They use a submodular objective function but solve a finite horizon planning problem as opposed to OP. However, none of these works account for potential failures of the robots, as we do in RMOP.

Recently, Jorgensen et al. introduced the Matroids Team Surviving Orienteers Problem (MTSO)~\cite{jorgensen2017matroid} which does account for individual robot failures. They assume that there is some given probability of failure associated with every edge in the environment. The goal is to maximize the expected rewards while ensuring each path satisfies some survival chance constraints. MTSO is appropriate when the failures of robots are random and follow a known distribution. The version we study, the RMOP, accounts for worst-case failures which makes it better suited when operating in adversarial conditions or in stochastic conditions when worst-case guarantees are sought due to unknown probability distributions.

Our work builds on recent work on robust submodular maximization~\cite{orlin2016robust,tzoumas2017resilient,tzoumas2018resilient,schlotfeldt2018resilient,zhou2018resilient,zhou2019distributed} which selects sets that are robust to worst-case removal of some subset of items. The challenge in this framework is to solve the trade-off between too much overlap, thereby not enough coverage (i.e., reward) and too little overlap, thereby not enough redundancy. The conceptual idea in these papers is similar --- the final solution consists of two subsets, one that has enough redundancy to ensure robustness against worst-case removal and the other that has enough coverage to get good overall performance. Orlin et al.~\cite{orlin2016robust} term the former as ``copies'' whereas it is called ``baits'' in \cite{zhou2018resilient}. The robust submodular maximization formulation has been applied for multi-robot, multi-target tracking in centralized~\cite{zhou2018resilient} and decentralized settings~\cite{zhou2019distributed} as well as for active information gathering with multi-robot teams~\cite{schlotfeldt2018resilient}. 

We seek similar robustness guarantees as in the works mentioned in the previous paragraph. {The key technical advancement we make is that these prior work solve a single-step selection problem whereas we solve a multi-step planning problem.} As a result, the single robot problem in the prior work can be trivially solved optimally (amounts to selecting the best amongst a finite set of options), whereas in the RMOP the single robot problem (OP) itself is NP-Hard. While both \cite{zhou2018resilient} and \cite{schlotfeldt2018resilient} use their results for planning over a finite horizon, they make key assumptions that are limiting.
{
Schlotfeldt~\cite{schlotfeldt2018resilient}  considers the continuous counterpart of the combinatorial problem considered in this paper and they formulate the problem under the optimal control framework with one key assumption that the single robot as well as multi-robot information gathering problem (without attacks) can be solved optimally (c.f. Proposition 1). Zhou et al.~\cite{zhou2018resilient} repeatedly solve the one-step problem at each time step. Instead, we show how to use an approximate solution to the OP to yield a bounded approximation solution to the combinatorial problem RMOP.
}
\subsection{Contributions} 
The main contributions of this paper are as follows.
    We introduce the Robust Multiple-Path Orienteering Problem and consider two specific types of this problem.
    For the offline case, We present a general approximation scheme to solve RMOP. 
    We analyze the running time and the performance of the algorithm. In particular, we show that the approximation ratio is a constant of the
    approximation factor for single robot OP. 
    We show how to employ three single robot algorithms for modular and submodular OP as subroutines in our algorithm and analyze their performance. {For the online case, we model the problem as a sequential two-player game and propose an MCTS-based method to solve the problem.
    We evaluate the performance of our algorithm using simulations involving a case study of information gathering in a tunnel with a team of robots.}
In addition, we give an alternative and more complete proof on the bound of the sequential algorithm in \cite{singh2009efficient} which is of independent interest.

{
A preliminary version of this paper was presented in RSS 2020 \cite{shi2020robust}. Compared to our previous work which includes only the offline RMOP, we also consider the online RMOP in this paper and propose an MCTS-based strategy to solve the problem. New simulation results are also provided for the online RMOP.
}

The rest of the paper is organized as follows. We provide the necessary background on submodular functions and formally introduce the offline RMOP in Sec. \ref{Problem Description}. Next, the approximation algorithm is proposed in Sec. \ref{Proposed Approach}. The online RMOP is introduced in Sec. \ref{RMOP_with_communication} and the corresponding algorithm is explained in Sec. \ref{section:mcts}. Detailed analysis of the proposed algorithm is given in Sec. \ref{Algorithm Analysis}. Finally, simulation results are given in Sec. \ref{Simulation} and \ref{sim:RMOP with communication}.

\section{Problem Description}\label{Problem Description}
In this section, we formally describe the Robust Multiple Orienteering Problem. We start by introducing notations and conventions used in the paper. 

We use calligraphic fonts to denote sets (e.g. $\mathcal{A}$). Given a set $\mathcal{A}$, $2^{\mathcal{A}}$ denotes the power set of $\mathcal{A}$ and $|\mathcal{A}|$ denotes the cardinality of $\mathcal{A}$. Given another set $\mathcal{B}$, the set $\mathcal{A} \setminus \mathcal{B}$ denotes the set of elements in $\mathcal{A}$ but not in $\mathcal{B}$. Given a set $\mathcal{V}$, a set function $f: 2^{\mathcal{V}} \mapsto \mathbb{R}$, and an element $x \in \mathcal{V}$, $f(x)$ is a shorthand that denotes $f(\{x\})$. We use $f_{\mathcal{A}}(\mathcal{B})$ to denote $f(\mathcal{A} \cup \mathcal{B})-f(\mathcal{A})$.

We now define two useful properties of set functions.
\begin{definition}[Normalized Monotonicity]
For a set $\mathcal{V}$, a function $f: 2^{\mathcal{V}} \mapsto \mathbb{R}$ is called as normalized, monotonically non-decreasing if and only if for any $\mathcal{A} \subseteq \mathcal{A}^{\prime} \subseteq \mathcal{V}, f(\mathcal{A}) \leq f(\mathcal{A}^{\prime})$ and $f(\mathcal{A})=0$ if and only if $A=\emptyset$.
\end{definition}
As a short-hand, we refer to a normalized, monotonically non-decreasing function as simply a monotone function.

\begin{definition}[Submodularity]
For a set $\mathcal{V}$, a function $f: 2^{\mathcal{V}} \mapsto \mathbb{R}$ is submodular if and only if for any sets $\mathcal{A} \subseteq \mathcal{V}$ and $\mathcal{A}^{\prime} \subseteq \mathcal{V}$, we have $f(\mathcal{A})+f(\mathcal{A}^{\prime}) \geq f(\mathcal{A} \cup \mathcal{A}^{\prime})+f(\mathcal{A} \cap \mathcal{A}^{\prime})$;
\end{definition}

Let $G(\mathcal{V},\mathcal{E})$ be a graph. 
A path $\mathcal{P}$ in $G$ is an ordered sequence of non-repeated vertices. As a shorthand, we use $\mathcal{P}$ to denote both the path (ordered set) as well as the unordered set of vertices along the path. When we use $\mathcal{P}$ as the path (ordered set), $\mathcal{P}(i)$ denotes $i_{th}$ vertex in $\mathcal{P}$. Let $\mathcal{T}=2^{\mathcal{V}}$ denote the power set of $\mathcal{V}$. Intuitively, $\mathcal{T}$ is the superset of all possible sets of vertices that a robot may visit along its path. The cost of a path $\mathcal{P}$, denoted by $C(\mathcal{P})$, is the sum of the edge weights along the path. We assume that the edge weights are metric. We study the \emph{rooted} version of the problem where the path for robot $i$, denoted by $\mathcal{P}_i$, must begin at a specific vertex $v_{s_i}$.

We consider the case that the \textit{reward function}, $g(\mathcal{P}): \mathcal{T} \rightarrow \mathbb{R}_{+}$, of a single robot is a monotone submodular function. We also study the special version where the function is modular (i.e., the reward of a path is the sum of rewards of the vertices along the path). 

Let $\mathcal{S}$ be some collection of $N$ paths corresponding to the $N$ robots in the team, $\mathcal{S}=\{\mathcal{P}_1,\mathcal{P}_2,\ldots,\mathcal{P}_N\}$. The team reward collected by any subset $\mathcal{S'}\subseteq\mathcal{S}$ is given by,
\begin{equation}
f(\mathcal{S'})=g\left(\bigcup_{\mathcal{P}_i\in \mathcal{S'}}\mathcal{P}_i\right).
\label{eqn:teamreward}
\end{equation}
Note that reward function of the team $f(\mathcal{S'})$ is a submodular function irrespective of whether the single robot reward function $g(\mathcal{P}_i)$ is submodular or not. Multiple robots can visit the same vertex but only one visit is accounted for when computing the reward of the team. That is, there can be no double counting of the rewards. 


We are now ready to formally define our problem. 
\begin{problem}[Offline RMOP]\label{main_problem}
Given a metric graph $G(\mathcal{V},\mathcal{E})$, {$N$ robots with starting positions $\{v_{s_1}, v_{s_2}, \ldots, v_{s_N}\}$, budget constraint $B$}, a robot reward function $g(\mathcal{P}): \mathcal{T} \rightarrow \mathbb{R}_{+}$, and a team reward function $f$ as defined in Equation~\ref{eqn:teamreward}, the offline Robust Multiple-Path Orienteering Problem seeks to find a collection of $N$ paths, $\mathcal{S}=\{\mathcal{P}_1,\ldots,\mathcal{P}_N\}$ that are robust to the worst-case failure of $\alpha$ robots:
\begin{equation}
\begin{aligned}
    \max_{\mathcal{S} \subseteq \mathcal{T}} & \min_{{\mathcal{A}} \subseteq \mathcal{S}}f(\mathcal{S} \setminus {\mathcal{A}}) \\
    \textrm{s.t.} \qquad & |{\mathcal{A}}| \leq \alpha, 0<\alpha<N \\
           & |\mathcal{S}| = N \\
           & C(\mathcal{P}_j) \leq B.
\end{aligned}
\end{equation}
where additionally $v_{s_j}$ must be the starting vertex when constructing a path $\mathcal{P}_j$ for robot $j$.
\end{problem}

The offline RMOP is suited to model the scenarios where we need to plan paths for all the robots before they are deployed and they cannot communicate with the base station to transmit collected rewards or with each other to replan during execution. Therefore, once a robot fails during the task, the reward of the whole path of that robot will be lost, which corresponds to the set removal of paths, and the team cannot adapt to the failures of robots. One practical example for the offline RMOP is the naval mine countermeasure mission \cite{sariel2008naval}, in which a team of robots is deployed to detect undersea mine information. In such a case, reliable communication is usually not available and the robot may fail due to mines or other adversaries. Mathematically,
the offline RMOP can be interpreted as two-stage perfect information sequential one-step game, where the first player (i.e., the team of robots) chooses a set $\mathcal{S}$, and the second player (i.e., the adversary), knowing $\mathcal{S}$, chooses a subset ${\mathcal{A}}$ to remove from $\mathcal{S}$. We seek worst-case guarantees --- in practice, the adversary may not know the paths for each robot. By playing against this stronger adversary, we guarantee that the performance against a weaker one will be even better. We evaluate this empirically by considering attack models other than the worst one.

The adversarial model considered in this paper is the same as that in prior work on robust submodular optimization~\cite{tzoumas2017resilient,tzoumas2018resilient,zhou2018resilient}. However, the offline RMOP is even harder since even at the single robot level, the optimization problem we need to solve (i.e., OP) is NP-Hard. Nevertheless, we present a constant-factor approximation algorithm for this problem next.

\section{Algorithm for offline RMOP}\label{Proposed Approach}
In this section, we present the general algorithm to solve the offline RMOP (Algorithm \ref{resilient_algorithm}). The algorithm uses a generic subroutine for solving OP. In the next section, we show examples of three subroutines that can be used and show how they affect the performance of the algorithm. 

Before we describe the algorithm, we present additional notation. If $\mathcal{S'}$ is a set of $N'\leq N$ paths, then let $\mathcal{R}(\mathcal{S'})$ denote the set of corresponding $N'$ robots whose paths are contained in $\mathcal{S'}$. We use $\mathcal{A}^*(\mathcal{S})\triangleq \arg\min_{\mathcal{A}}f(\mathcal{S}\setminus{\mathcal{A}})$ to denote the worst-case set of paths that are removed from a given set of path $\mathcal{S}$. Therefore, $\mathcal{S} \setminus {\mathcal{A}}^{*}(\mathcal{S})$ denotes the set of paths that are not attacked from $\mathcal{S}$ with $|\mathcal{A}^{*}(\mathcal{S})| \leq \alpha$. 

\begin{algorithm}[ht]\label{resilient_algorithm}
    \caption{Algorithm for Problem \ref{main_problem}}
    \SetKwInOut{Input}{Input}
    \SetKwInOut{Output}{Output}
    \Input{
    Per problem \ref{main_problem} requires following inputs:
    \begin{itemize}
        \item set of robots $\mathcal{R} = \{1,\ldots,N\}$
        \item metric graph $G$
        \item starting vertices $\{v_{s_1},v_{s_2},\ldots,v_{s_N}\}$
        \item number of maximum potential attacks $\alpha$ and budget $B$
    \end{itemize}
    }
    \Output{Set $\mathcal{S}$ of paths for each robot}
    $\mathcal{S}_1 \gets \emptyset,\mathcal{S}_2 \gets \emptyset, \mathcal{M} \gets \emptyset$  \\
    \For{$i \gets 1$ \KwTo $N$}{
    $\mathcal{P}_i \gets OP(G, v_{s_i}, B)$\\
    $\mathcal{M} \gets \mathcal{M} \cup \{\mathcal{P}_i\}$\\
    }
    \textit{flag} $\leftarrow$ True\\
    \While{flag}{
    Sort elements in $\mathcal{M}$ such that $\Tilde{\mathcal{M}}=\{\mathcal{P}_1^{\prime},\mathcal{P}_2^{\prime},\ldots,\mathcal{P}_N^{\prime}\}$ and $f(\{\mathcal{P}_1^{\prime}\}) \geq f(\{\mathcal{P}_2^{\prime}\}) \geq \ldots \geq f(\{\mathcal{P}_N^{\prime}\})$\\
    $\mathcal{S}_1 \gets \{\mathcal{P}_1^{\prime},\mathcal{P}_2^{\prime},\ldots,\mathcal{P}_{\alpha}^{\prime}\}$\\ 
    //extract starting positions for the rest of robots
    $\Tilde{v}_s \gets \{v_{s_j} | \forall j \in \mathcal{R} \setminus \mathcal{R}(\mathcal{S}_1)\}$\\
    //Sequential greedy assignment 
    $\mathcal{S}_2 \gets SGA(G, \Tilde{v}_s, B)$\\
    
    //while loop control \\
    \eIf{$f({\mathcal{P}_i}) \geq f({\mathcal{P}_j}), \forall \mathcal{P}_i \in \mathcal{S}_1, \mathcal{P}_j \in \mathcal{S}_2$}{
      $flag \gets$ False
    }
    {
      Find all robots $j \in \mathcal{R}(\mathcal{S}_2)$ such that \\
      $\exists i \in \mathcal{R}(\mathcal{S}_1),  f(\{\mathcal{P}_j \in \mathcal{S}_2\}) > f(\{\mathcal{P}_i \in \mathcal{S}_1\})$ \\
      Replace the path stored in $\mathcal{M}$ corresponding to robot $j$ with the better path found when constructing $\mathcal{S}_2$ 
    }
    
    }
    $\mathcal{S} \gets \mathcal{S}_1 \cup \mathcal{S}_2$
\end{algorithm}



The algorithm consists of two main steps: first, it calls a subroutine for solving OP $N$ times to compute a path for each robot independently. It then chooses $\alpha$ paths (denoted by $\mathcal{S}_1$) with the highest individual rewards without considering overlap with other robots; Second, it uses sequential greedy assignment to find paths for the rest of the robots (denoted by $\mathcal{S}_2$) by querying the OP subroutine $N-\alpha$ times. The \textbf{while} loop is used to maintain an invariant that the paths in $\mathcal{S}_1$ are always better than the paths in $\mathcal{S}_2$. 

As described earlier, there is a tradeoff between redundancy and coverage in RMOP. The two sets of paths are constructed so that $\mathcal{S}_1$ adds redundancy and $\mathcal{S}_2$ adds coverage, together yields a provably good solution for RMOP. We explain each step in Algorithm~\ref{resilient_algorithm} next.


\paragraph*{Constructing $\mathcal{S}_1$}
Each of the $\alpha$ paths in $\mathcal{S}_1$ are  better than those in $\mathcal{S}_2$. The paths in $\mathcal{S}_1$ may overlap with each other and also overlap with those in $\mathcal{S}_2$. Thus, these paths serve to add redundancy to the team.
Constructing the best $\alpha$ paths with respect to $f$ itself is NP-hard. Therefore, Algorithm \ref{resilient_algorithm}firstly solves the orienteering problem for each robot independently and stores in $\mathcal{M}$ the (approximately optimal) paths for individual robots (lines 2--5). Then Algorithm \ref{resilient_algorithm}sorts the paths in $\mathcal{M}$ based on their collected rewards (line 8) and chooses the $\alpha$ best paths to be $\mathcal{S}_1$ (line 9). 

\paragraph*{Constructing $\mathcal{S}_2$}
After finding $\mathcal{S}_1$, Algorithm \ref{resilient_algorithm} needs to find the best paths for the rest of robots $\mathcal{R} \setminus \mathcal{R}(\mathcal{S}_1)$. Unlike $\mathcal{S}_1$, here the algorithm explicitly considers overlap when finding the paths. Thus, $\mathcal{S}_2$ serves to add coverage to the solution.

However, selecting optimal paths for $\mathcal{R} \setminus \mathcal{R}(\mathcal{S}_1)$ is a multiple-path orienteering problem and is also NP-hard. Therefore, Algorithm \ref{resilient_algorithm} approximates the solution by employing the sequential greedy algorithm  (line 11). For completeness, we present the pseudocode for SGA in Algorithm~\ref{SGA_algorithm}. 


Specifically, for robots in $\mathcal{R} \setminus \mathcal{R}(\mathcal{S}_1)$, Algorithm \ref{SGA_algorithm} finds a path using an approximation algorithm for OP (line 4). Then, Algorithm \ref{SGA_algorithm} sets the reward for the vertices visited by that robot to be zero (lines 6--8). This process repeats until we find a path for all robots. Here we implicitly assume that there is at least one path for each robot satisfying the budget constraints. 

The paths in $\mathcal{S}_1\cup \mathcal{S}_2$ form the solution to RMOP. However, we also have an outer \textbf{while} loop which we explain next.

\paragraph*{Invariant} Our analysis requires the paths in $\mathcal{S}_1$ and $\mathcal{S}_2$ to have the following property: $f(\{\mathcal{P}_i\}) \geq f(\{\mathcal{P}_j\}), \forall \mathcal{P}_i \in \mathcal{S}_1, \mathcal{P}_j \in \mathcal{S}_2$. 
This condition is trivially met if the single robot problem has to just choose the best amongst a fixed set of trajectories as in the prior work~\cite{tzoumas2017resilient,zhou2018resilient}. However, when solving RMOP, we employ a subroutine for solving OP which gives us the paths in $\mathcal{S}_1$ and $\mathcal{S}_2$. Since the subroutine uses an approximation algorithm for OP instead of an exact optimal one, we cannot guarantee that this invariant holds. For example, if the subroutine uses randomness, then running the same algorithm twice may give different results. In any case, all we can guarantee is that the paths found by the subroutine will be no more than a constant from the optimal.

We fix this problem by utilizing a \textbf{while} loop ({lines 7--20}). When the condition of the \textbf{while} loop holds (lines 13--15), the loop flag is set to be false and the while loop terminates. Otherwise (lines 16--19),  Algorithm \ref{resilient_algorithm} will find those robots that violate the above inequality and update their corresponding paths in the set $\mathcal{M}$. Recall that $\mathcal{M}$ is used to store the best path corresponding to each robot. Then, while loop will restart to construct $\mathcal{S}_1$ using the updated $\mathcal{M}$ and $\mathcal{S}_2$ for the remaining robots, again. We show that this loop will eventually terminate.

\begin{corollary}\label{corollary: while_loop}
The \textbf{while} loop in Algorithm \ref{resilient_algorithm} will terminate in a finite number of steps.
\end{corollary}
\begin{proof}
If the flag is not set to false after an iteration of the \textbf{while} loop, then it must mean that at least one new path found when constructing $\mathcal{S}_2$, say for robot $j$, is better than some path in $\mathcal{S}_1$. Suppose this better path is $\mathcal{P}_j^{'}$. Note that the set $\mathcal{M}$ includes a path for the robot $j$, say $\mathcal{P}_j$. Since $\mathcal{S}_1$ consists of the best $\alpha$ paths in $\mathcal{M}$ and $\mathcal{P}_j \notin \mathcal{S}_{1}$, then it must mean that the path $\mathcal{P}_j^{'}$ is strictly better than $\mathcal{P}_j$. Thus, after every iteration of the \textbf{while} loop, if the flag is not set, then at least one path in $\mathcal{M}$ has improved. For each robot given a fixed budget, there is a maximum amount of reward that it can collect. We cannot keep increasing the rewards of paths in  $\mathcal{M}$. Therefore, the while loop must terminate after finite iterations.  
\end{proof}
\begin{algorithm}[ht]\label{SGA_algorithm}
    \caption{Sequential Greedy Assignment}
    \SetKwInOut{Input}{Input}
    \SetKwInOut{Output}{Output}
    \SetKwProg{Fn}{Function}{:}{}
    \SetKwFunction{SGA}{SGA}
    \Fn{\SGA{$G, v_{s}, B$}}{ 
    \Input{\begin{itemize}
        \item A graph $G$ representing environment
        \item Budget $B$ for each robot 
        \item Starting positions $v_s$
    \end{itemize}}
    \Output{a collection $\mathcal{A}$ of paths }
    $\mathcal{A} \gets \emptyset, G^{\prime} \gets G, N \gets length(v_s)$  \\ 
    \For{$j \gets 1$ \KwTo $N$ }{
    $\mathcal{P}_j \gets OP(G^{\prime}, v_{s_j}, B)$\;
    $\mathcal{A} \gets \mathcal{A} \cup \{\mathcal{P}_j\}$ \\
    \ForEach{$v \in \mathcal{P}_{j}$}{
     Set the reward of $v \in G^{\prime}$ to be zero \\
    }
    }
    \textbf{return} $\mathcal{A}$
    }
    \textbf{end}
\end{algorithm}
\begin{remark}\label{S1greaterS2}
In practice, the loop in Algorithm \ref{resilient_algorithm} typically terminates after just one iteration. Paths in $\mathcal{S}_1$ are found without considering overlap. On the other hand, when solving SGA the robots find their paths by taking into account overlap with the previously found paths. The conditions in the latter are a subset of the former. Furthermore, none of the three subroutines that we employ for OP include any randomness. Therefore, it is unlikely that the paths in $\mathcal{S}_2$ will be better than that in $\mathcal{S}_1$. As such, it is less likely that the \textbf{while} loop will take more than one iteration. Nevertheless, we give the full algorithm for completeness.

\end{remark}

So far, we have not discussed the subroutine used to solve OP. In the Sec. \ref{Algorithm Analysis}, we present the analysis of the algorithm and then present the three subroutines.

\section{{Online RMOP}}\label{RMOP_with_communication}
In the offline RMOP, we consider the scenario in which robots cannot communicate with each other or the base station when they execute tasks. As a result, once a robot fails or is attacked we will lose all the rewards collected by that robot i.e., lose the path and we correspondingly plan for the worst-case attacks. In the online RMOP, we consider the scenario where robots can communicate with the base station to send collected rewards during execution and with teammates to replan in response to the attacks.  Attackers' behaviors are assumed to be the same: there are $\alpha$ attacks in total and the attacks can happen at any nodes in the map. But the difference is that for the offline case, it doesn't matter when the robot fails because as long as it fails we will lose the whole path. For the online case, when the robot fails will influence how much reward robots can collect. For example, if a robot is attacked at the very first node, we can get only get the reward of that particular node; but if the robot is attacked when it almost uses up the budget, we can collect most reward along its path. Therefore, the attacker's behaviors are characterized by two types of decision variables: when to attack and which to attack. Without loss of generality, we assume that it takes one unit of time (not necessarily the same amount of budget) to traverse one edge of the graph such that robots are able to replan synchronously. Such a graph can be obtained by carefully designing motion primitives and discretizing the environment or by adding some virtual nodes on the edges of a graph. With this assumption, the online RMOP can be formulated as follows.

\begin{definition}[Attacker behavior set]
 An attack behavior set $\mathcal{A}=\{(t_1, r_1), \ldots, (t_{\beta}, r_{\beta})\}$ is a set of tuples, each of which consists of two elements: the first element indicates when to attack and the second element indicates which robot to attack.
\end{definition}
\begin{problem}[Online RMOP]\label{online_problem}
Given a metric graph $G(\mathcal{V},\mathcal{E})$, {$N$ robots with starting positions $\{v_{0}^1, v_{0}^2, \ldots, v_{0}^N\}$, budget constraint $B$}, a robot path reward function $g(\mathcal{P}): \mathcal{T} \rightarrow \mathbb{R}_{+}$, the online Robust Multiple-Path Orienteering Problem seeks to find a collection of $N$ paths 
\begin{equation*}
        \mathcal{S}=\left\{
        \begin{array}{l}
        \mathcal{P}_1=\{v_{0}^1, ~\ldots, v_{m}^1, \ldots, ~v_{\scaleto{end}{3pt}}^1\}\\
        \vdots \\
        \mathcal{P}_N=\{v_{0}^N, ~\ldots, v_{m}^N, \ldots, ~v_{\scaleto{end}{3pt}}^N\}
        \end{array}
        \right\}
\end{equation*}
that are robust to the worst-case failure of $\alpha$ robots:
\begin{equation}
\begin{aligned}
    \max_{\mathcal{S} \subseteq \mathcal{T}} & \min_{{\mathcal{A}=\{(t_j, r_j)\}} }~g(\bigcup_{i=1}^{N}\mathcal{P}_i \setminus \bigcup_{j=1}^{|{\mathcal{A}}|} \mathcal{P}_{r_j}[t_{j}+1 : end]) \\
    \textrm{s.t.} \qquad & |{\mathcal{A}}| \leq \alpha, 0<\alpha<N \\
           & |\mathcal{S}| = N \\
           & C(\mathcal{P}_j) \leq B,
\end{aligned}
\end{equation}
where $\mathcal{A}=\{(t_1, r_1), \ldots, (t_{|{\mathcal{A}}|}, r_{|{\mathcal{A}}|})\}$ is an attacker behavior set; $\mathcal{P}_{r_j}[t_{j}+1 : end]$ is the path segment of robot $r_j$'s path $\mathcal{P}_{r_j}$ from the node $\mathcal{P}_{r_j}(t_{j}+1)$ to the node $\mathcal{P}_{r_j}(end)$ ; additionally $v_{0}^j$ must be the starting vertex when constructing a path $\mathcal{P}_j$ for robot $j$.
\end{problem}

Intuitively, in Problem \ref{online_problem}, we want to find paths for robots while $\alpha$ failures in total can happen at any nodes along the paths. If a robot fails at a particular node $\mathcal{P}(t)$, then it cannot collect reward after that node anymore. We model this problem as a discrete, sequential, two-player zero-sum game between the attackers and the robots. Since robots can communicate with each other, we aim to find one adaptive planning strategy that can adapt to the attacks. 

\begin{prop}
It's not the optimal strategy for attackers to always launch all attacks at the very first step.
\end{prop}
\begin{proof}
One example is given in Fig. \ref{fig:rational_attacker}.
\end{proof}
It should be noted that a rational attacker will not always launch attacks at the first step. For example, in Fig. \ref{fig:rational_attacker}, there are four robots that are initially located in node $a$ and each of them has a budget of $B=4$. There are $\alpha=2$ attackers in the environment. If the attackers attack two robots at the first step. The survived two robots can certainly collect 35 reward in total by planning path $(a,b,c,d,e)$ and $(a,b,c,f,g)$. By contrast, if the attackers choose to wait until they know how robots move after node $c$, the robots can collect at most 25 rewards. Here is the explanation. After four robots reach node $c$ and remain all survived, the best strategy for robots is to send three of them to follow the path $(c,d,e)$ and another robot to follow $(c,f,g)$ considering that there are still $\alpha=2$ attacks. In the worst-case where the robot following $(c,f,g)$ and one robot following $(c,d,e)$  got attacked, robots can collect 25 reward in total. If robots don't adopt the best strategy, they will get less than 25 rewards in the worst case. In the next section, we demonstrate how to find the optimal solution for this game using two-player MCTS.   

\begin{figure}[htbp]
\centerline{\includegraphics[scale=0.9]{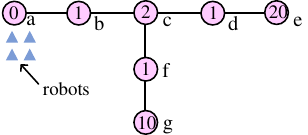}}
\caption{One example to show that the rational attackers will not always launch attacks at the first step. Four robots start from node $a$ and each has a budget of 4. There are $\alpha =2$ attacks in the environment. If attackers launch all attacks at the first step, the robots can collect 35 rewards in total by planning path $(a,b,c,d,e)$ and $(a,b,c,f,g)$ while robots can collect at most 25  if the attackers choose to wait until they know how robots move after node $c$.
}
\label{fig:rational_attacker}
\end{figure}
\section{MCTS for online RMOP}\label{section:mcts}
MCTS is an approach for finding optimal actions by randomly selecting samples from search space and incrementally building the search tree \cite{browne2012survey} and is widely applied in robotics applications, including scouting \cite{zhang2021game}, active parameter estimation \cite{slade2017simultaneous}, environment monitoring \cite{marchant2014sequential}, and multi-robot active perception \cite{best2019dec}. As shown in Fig. \ref{fig:MCTS}, there are four basic steps in each iteration of an MCTS process: selection, expansion, simulation, and backpropagation \cite{chaslot2008monte}.  

In the paper, we model the Problem \ref{online_problem} as a sequential, discrete two-player game and we adopt the MCTS algorithm to solve the game in an online fashion. At each step, attackers use their strategies to take one action first, and then robots take one action. Intuitively, it means that attackers observe the states of robots to decide to attack or not, and then robots respond to that. To choose one action, the robots will incrementally grow the search tree with some computational budget and then select one action to take from the root of the search tree based on the average reward of each action. Such a process continues until all robots run all of the budget. 

When our algorithm grows the search tree, the state of each robot is a tuple $s=(v, I)$ where $v \in \mathcal{V}$ represent the current position of the robot and $I$ is an indicator on whether the robot has been attacked ($I=1$) or not ($I=0$). The joint state of the team is the product of the individual states of robots and is stored in each node of the search tree. Robots and attackers alternate turns in growing the tree. When it's robots' turn, they will decide the transition of the positions while the attackers can decide the value of the indicator state in attackers' turn. Once one indicator state is set to be one, which means a robot is attacked, it will remain to be one for the rest of the game and that robot cannot move anymore i.e, the only action that robot can take in the game is to stay there. Similarly, if a robot has run out of budget, the only action available is to stay at the current position.  It should be noted that we present Algorithm \ref{MCTS_algorithm} from the perspective of robots but attackers can also use similar strategies. In the following, we refer to two-player MCTS (if there are two alternating turns in the search) as MCTS with adversaries and one-player MCTS without considering opponents as naive MCTS. Details of the Algorithm \ref{MCTS_algorithm} are given below.

\begin{enumerate}
    \item \textit{Selection} (Line 4 in Algorithm \ref{MCTS_algorithm}; Line 1-11 in Algorithm \ref{MCTS_subroutine}): Starting from the root node, a selection procedure is recursively applied until some leaf node is reached. In each recursion, a child node is selected based on Upper Confidence Bound for Trees (UCT) Kocsis and Szepesvári \cite{kocsis2006bandit}. There are two parts to the UCT value: exploitation and exploration. The exploitation part corresponds to the average rollout reward obtained and the exploration part is decided by the number of times that the node has been visited ($n(v^{\prime})$) and the number of times that the current node's parent has been visited ($n_p$). If a node is less visited, the exploration value will increase which encourages the selection of that node. It should be noted that if it's robots' turn robots will select the node with the highest UCT value (Line 6) while the attacker will select the node with the lowest UCT value if it's the attacker's turn (Line 8).  
    \item \textit{Expansion} (Line 5 in Algorithm \ref{MCTS_algorithm}; Line 12-19 in Algorithm \ref{MCTS_subroutine}): One (or more) child nodes are added to the tree based on the available actions. If the node has reached the terminal level, e.g., run out of budget, the current node will be returned (Line 13-15). Otherwise, add all children of the current node to the tree and return the first child node (Line 15-18).
    \item \textit{Simulation} (Line 6 in Algorithm \ref{MCTS_algorithm}; Line 20-29 in Algorithm \ref{MCTS_subroutine}): A rollout is conducted from the chosen node using the default policy until some terminal condition is met. The obtained reward will be returned.
    \item \textit{Backpropagation} (Line 7 in Algorithm \ref{MCTS_algorithm}; Line 30-34 in Algorithm \ref{MCTS_subroutine}): The simulation result is propagated back to the root and update node information along the propagation path. 
\end{enumerate}

\begin{figure}[htbp]
\centerline{\includegraphics[scale=0.9]{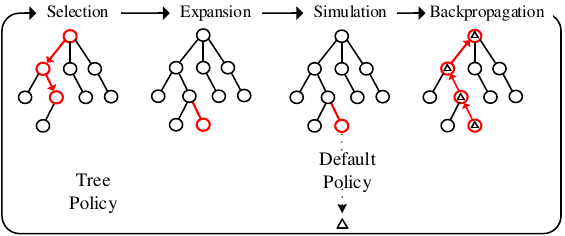}}
\caption{One iteration of the general MCTS approach \cite{browne2012survey}.
}
\label{fig:MCTS}
\end{figure}

\begin{algorithm}[ht]\label{MCTS_algorithm}
    \caption{Monte Carlo Tree Search}
    \SetKwInOut{Input}{Input}
    \SetKwInOut{Output}{Output}
    \SetKwProg{Fn}{Function}{:}{}
    \SetKwFunction{FMain}{MCTS}
    \Fn{ \FMain{$s_1 , s_2 , \ldots , s_N$} }{
    \Input{
    Initial states of robots, which includes information on attacks.
    }
    \Output{A search tree}
    Create a \textit{tree} with root node $v^t_0$ with initial states  $(s_1 , \ldots , s_N)$\\ 
    \While{computational budget not used up}{
    \nonl // selection \\
    $v_{sel}^t \gets ~\text{Selection}(tree, ~v^t_0)$ \\
    \nonl // expansion \\
    $v_{exp}^t \gets ~\text{Expansion}(tree, ~v^t_{sel})$ \\
    \nonl // rollout \\
    $Reward \gets ~\text{Simulation}(tree, ~v^t_{exp})$ \\
    \nonl // backpropagation \\
    $\text{Backpropagation}(tree, ~Reward, ~v_{exp}^t)$
    }
    \textbf{return} $tree$
    }
    \textbf{end} \\
\end{algorithm}

\begin{algorithm}[ht]\label{MCTS_subroutine}
    \caption{Monte Carlo Tree Search Subroutines}
    \SetKwFunction{Selection}{Selection}
    \SetKwProg{Fn}{Function}{:}{}
    \Fn{\Selection{$tree, ~v$}}{
    \If{level($v$) = TERMINAL}{
    \textbf{return} $v$
    } 
    \eIf{turn($v$)~=~ROBOT}{
    \nonl // $n_p$ is the number of times that the parent of $v$ has been visited \\
        $$v \gets \argmax_{v^{\prime} \in \text{children}(v)} ~ \frac{Q(v^{\prime})}{n(v^{\prime})} + c\sqrt{\frac{2\ln{n_p}}{n(v^{\prime})}}$$ 
    }
    {    
    $$v \gets \argmin_{v^{\prime} \in \text{children}(v)} ~ \frac{Q(v^{\prime})}{n(v^{\prime})} - c\sqrt{\frac{2\ln{n_p}}{n(v^{\prime})}}$$ 
    }
    \textbf{return} \Selection{$tree, ~v$}
    }
    \textbf{end} \\
    \nonl ~ \\
    \SetKwFunction{Expansion}{Expansion}
    \Fn{\Expansion{$tree, ~v$}}{
    \eIf{level($v$) = TERMINAL}{
    return $v$
    }{
    Add all child nodes of $v$ to $Tree$ \\
    \textbf{return} the first child node
    }
    }
    \textbf{end} \\
    \nonl ~ \\
    \SetKwFunction{Simulation}{Simulation}
    \Fn{\Simulation{$tree, ~v$}}{
    \While{level($v$) $\neq$ TERMINAL}{
    \eIf{turn($v$)=ROBOT}{
    $v \gets$ RobotDefaultPolicy($v$)
    }{
    $v \gets$ AttackerDefaultPolicy($v$)
    }
    }
    \textbf{return} CollectReward
    }
    \textbf{end} \\
    \nonl ~ \\
    \SetKwFunction{Backpropagation}{Backpropagation}
    \Fn{\Backpropagation{$tree, ~Reward,~v$}}{
    \While{$v \neq $ NULL}{
    \nonl // update total reward value
    $tree.v.Q \gets tree.v.Q + Reward$ \\ 
    $tree.v.n \gets tree.v.n + 1$ \\
    
    }
    }
    \textbf{end}
\end{algorithm}

\section{Performance Analysis}\label{Algorithm Analysis}
In this section, we quantify the performance of the proposed Algorithm \ref{resilient_algorithm}. We first present a new analysis for the Sequential Greedy Assignment (SGA) and then show the performance bound for our algorithm. The performance is based on the notion of curvature of the set functions.

\begin{definition}[Curvature]\label{curvature}
 Consider a finite ground set $\mathcal{V}$ and a monotone submodular set function $h:2^{\mathcal{V}} \mapsto \mathbb{R}$. The curvature of $h$ is defined as,
 \begin{equation}
     k_h = 1-\min_{v \in \mathcal{V}}\frac{h(\mathcal{V})-h(\mathcal{V} \setminus v)}{h(v)}
 \end{equation}
 \end{definition}
 The curvature takes values $0 \leq k_h \leq 1$ and measures how far $h$ is from modularity. When $k_h=0$, $h$ is modular since for all $v \in \mathcal{V}$, we get $h(\mathcal{V})-h(\mathcal{V}\setminus \{v\})=h(v)$. On the other extreme, when $k_h=1$ there exists some element $v$ that makes no unique contribution to the rest of the set, since we get $h(\mathcal{V})=h(\mathcal{V}\setminus \{v\})$. We assume that the curvature $k_g$ of the reward function and that $k_f$ of objective function is strictly less than 1. This is reasonable since it implies every vertex and path in the environment makes some non-zero  unique contribution over the rest.

We first analyze the SGA and then use that analysis for proving the performance bound of our algorithm.
\subsection{Sequential Greedy Assignment}

SGA was first proposed in \cite{singh2009efficient} to solve the MOP. Note that the MOP is the same as RMOP if we consider $\alpha=0$. SGA solves the problem by finding the path for the $i^{th}$ robot in the $i^{th}$ iteration, by considering the paths found in the previous $i-1$ iterations.  

Let $\mathcal{Q}^{*}=\{\mathcal{Q}_1^{*}, \mathcal{Q}_2^{*}, \ldots, \mathcal{Q}_M^{*}\}$ be the optimal solution to the MOP with $M$ robots. It is easy to see that these paths will be non-overlapping in a metric graph. 
\begin{prop}
\label{no_overlap_assumption}
There exists an optimal solution for the MOP consisting of no overlapping paths, i.e., $\mathcal{Q}_{i}^{*} \cap \mathcal{Q}_{j}^{*}=\emptyset,  \forall i \neq j$.
\end{prop}
The proof is given in the supplementary document.

We use an approximation algorithm to find the path for each robot. Let $\mathcal{P}=\{\mathcal{P}_1, \mathcal{P}_{2},\ldots,\mathcal{P}_M\}$ be the set of paths returned by SGA. Note that SGA runs for $M$ iterations. Let $\mathcal{X}_{1:j}$ denote the set $\{\mathcal{P}_1, \mathcal{P}_2, \ldots,  \mathcal{P}_{j}\}$, which is the collection of paths returned by SGA after the first $j$ iterations. We have $\mathcal{X}_{1:0}=\emptyset$. 
With slight abuse of notation, we use $\mathcal{X}_{1:j}$ to refer to the (unordered set of) vertices visited by the paths in $\mathcal{X}_{1:j}$.

{
Every iteration of SGA requires solving an NP-Hard problem (c.f. line 4 in Algorithm \ref{SGA_algorithm}). Let $\mathcal{P}_j^{*}$ be the optimal solution for the problem in iteration $j$, i.e.,
$$\mathcal{P}_j^{*} = \argmax_{\pi_j \in \Pi_j}~f_{\mathcal{X}_{1:j-1}}(\pi_j)$$
where $\Pi_j$ is the set of all feasible paths for robot $j$ and $f_{X}(\mathcal{P}) \triangleq f(X \cup \mathcal{P})-f(X)$. Let $\mathcal{O}_j$ be the set $\{\mathcal{P}_1^{*}, \mathcal{P}_2^{*}, \ldots, \mathcal{P}_{j}^{*}\}$, which is the collection of optimal paths for the problems in iterations 1 through $j$. We also set $\mathcal{O}_0=\emptyset$. 
}

{The analysis in \cite{singh2009efficient} gives the relation between $f(\mathcal{X}_{1:M})$ and $f(\mathcal{O}_M)$ which is not necessarily be the same as $f(\mathcal{Q}^{*})$. In fact, the underlying relationship between $f(\mathcal{O}_M)$ and global optimal $f(\mathcal{Q}^{*})$ is unclear. Recall that $\mathcal{Q}^{*}$ is the optimal set of $M$ paths for MOP and $\mathcal{O}_M$ is the set of $M$ paths where the $j^{th}$ path is the optimal solution to the problem in iteration $j$ of SGA. These are not necessarily the same. One way to see this is to note that the order of the robots in SGA is arbitrary. If we shuffle the order in which the robots select their paths, then the paths found for each robot as well as the per stage optimal paths $\mathcal{O}_M$ will change but the optimal solution for MOP $\mathcal{Q}^{*}$ will still be the same.
In the following, we will directly establish the relation between $f(\mathcal{X}_{1:M})$ and $f(\mathcal{Q}^{*})$. Our proof uses ideas from \cite{blum2003approximation}.
}

For now, assume that we use an $\eta$ approximation for OP,
$$f_{\mathcal{X}_{1:j-1}}(\mathcal{P}_j) \geq \frac{1}{\eta}f_{\mathcal{X}_{1:j-1}}(\mathcal{P}_j^{*}).$$

\begin{theorem}\label{SGA_theorem}
	Algorithm \ref{SGA_algorithm} (SGA) gives a $\frac{1+\eta}{1-k_g}$ approximation for MOP, where $\eta$ is the approximation factor for OP and $k_g$ is the curvature for the reward function $g(\cdot)$.
\end{theorem}
\begin{proof}
Let $\Delta_{i}\triangleq\mathcal{X}_{1:i-1} \cap \mathcal{Q}_{i}^{*}$ be the set of vertices visited by both the optimal path for the robot $i$ and paths found using SGA for robots 1 through $i-1$. Let $\Delta\triangleq\bigcup_{i}\Delta_{i}$. When we construct a path for robot $i$ in the $i$-th iteration, there is a feasible path for robot $i$ that visits all vertices in $\mathcal{Q}_{i}^{*}\setminus\Delta_{i}$. That is, if we remove the vertices from the optimal path $\mathcal{Q}_{i}^{*}$ for robot $i$ that are also in $\mathcal{X}_{1:i-1}$, the remaining vertices in $\mathcal{Q}_{i}^{*}$ still form a feasible path for robot $i$ (since it cannot be longer). 
Therefore, the optimal path $\mathcal{P}_{i}^{*}$ in the iteration $i$ for robot $i$ should satisfy 
\begin{equation}\label{SGA_step1}
\begin{split}
      f_{\mathcal{X}_{1:i-1}}(\{\mathcal{P}_{i}^{*}\}) &
     \geq f_{\mathcal{X}_{1:i-1}}(\{\mathcal{Q}_{i}^{*} \setminus \Delta_{i}\})\\
     &= f(\mathcal{X}_{1:i-1} \cup \{(\mathcal{Q}_{i}^{*} \setminus \Delta_{i})\})-f(\mathcal{X}_{1:i-1}) \\
     &= g(\cup_{j=1}^{i-1}\mathcal{P}_{j} \cup (\mathcal{Q}_{i}^{*} \setminus \Delta_{i})) - g(\cup_{j=1}^{i-1}\mathcal{P}_{j})
\end{split}
\end{equation}
Given two sets $\mathcal{Y}, \mathcal{Z} \subseteq \Pi$, 
using the inequality presented in the footnote in \cite{sviridenko2017optimal}, we have
\begin{equation}\label{key_inequality}
    g(\mathcal{Z} \cup \mathcal{Y})-g(\mathcal{Z})+\sum_{j \in \mathcal{Y} \cap \mathcal{Z}}g_{\mathcal{Z} \cup \mathcal{Y}\setminus \{j\}}(j) \geq (1-k_g)g(\mathcal{Y}).
\end{equation}
Note that $\cup_{j=1}^{i-1}\mathcal{P}_{j} \cap \{\mathcal{Q}^{*}_{i}\setminus\Delta_{i}\}=\emptyset$. Applying inequality \eqref{key_inequality} to \eqref{SGA_step1} and using submodularity, we have 
   \begin{align}
        f_{\mathcal{X}_{1:i-1}}(\{\mathcal{P}_{i}^{*}\}) &\geq    f_{\mathcal{X}_{1:i-1}}(\{\mathcal{Q}^{*}_{i}\setminus\Delta_{i}\})\\
       &\geq   (1-k_g) g(\mathcal{Q}^{*}_{i}\setminus\Delta_{i}) \\
       &\geq   (1-k_g)(g(\mathcal{Q}_{i}^{*})- g(\Delta_{i})).
   \end{align}
 Using the $\eta$ approximation algorithm for OP, we have $$ f_{\mathcal{X}_{1:i-1}}(\{\mathcal{P}_i\}) \geq \frac{1}{\eta}(1-k_g)(g(\mathcal{Q}_{i}^{*})-g(\Delta_{i})).$$

Summing over all $i$, we get 
\begin{equation}
    \sum_{i=1}^{M} \frac{\eta}{1-k_g} f_{\mathcal{X}_{1:i-1}}(\{\mathcal{P}_{i}\}) \geq \sum_{i=1}^{M}g(\mathcal{Q}_{i}^{*})-\sum_{i=1}^{M}g(\Delta_{i}).
\end{equation}
The left hand side is equal to $\frac{\eta}{1-k_g} f(\mathcal{X}_{1:M})$ by definition. Furthermore, by submodularity,
\begin{equation}\label{RHS_first_term}
    \sum_{i=1}^{M}g(\mathcal{Q}_{i}^{*}) \geq g(\bigcup_{i=1}^{M} \mathcal{Q}_{i}^{*}).
\end{equation}
We also have, 
\begin{align}
    \begin{split}\label{eq_lemma1}
        g(\Delta) &\geq (1-k_g)\sum_{\delta \in \Delta}g(\{\delta\})
    \end{split}\\
    \begin{split}\label{eq_assumption}
        & = (1-k_g)\sum_{i=1}^{M}\sum_{\delta \in \Delta_{i}}g(\{\delta\})
    \end{split}\\
    \begin{split}\label{eq_submodular}
        & \geq (1-k_g)\sum_{i=1}^{M}g(\Delta_{i})
    \end{split}
\end{align}
where Eq. \eqref{eq_lemma1} holds due to Lemma 1 in \cite{tzoumas2017resilient}; Eq. \eqref{eq_assumption} follows from Proposition \ref{no_overlap_assumption} that states $\Delta_{i} \cap \Delta_{j}=\emptyset,  \forall i \neq j$; Eq. \eqref{eq_submodular} is due to submodularity of $g$. Rearranging the terms and using the definition of $\Delta$,
\begin{equation}\label{eq_summary}
    -\sum_{i=1}^{M}g(\Delta_{i}) \geq -\frac{1}{1-k_g}g(\bigcup_{i=1}^{M} \Delta_{i}).
\end{equation}

Since $\Delta_{i} \in \mathcal{X}_{1:i-1}$, by monotonicity we have
\begin{equation}\label{eq_monotonoe}
    g(\bigcup_{i=1}^{M} \Delta_{i}) \leq g(\bigcup_{i=1}^{M} \mathcal{X}_{1:i-1}) \leq g(\mathcal{X}_{1:M})=g(\bigcup_{i=1}^{M} \mathcal{P}_{i}).
\end{equation}
Using Eq. \eqref{RHS_first_term}, \eqref{eq_summary}, and \eqref{eq_monotonoe}, we have,
$$\frac{\eta}{1-k_g} f(\mathcal{X}_{1:M}) \geq g(\bigcup_{i=1}^{M} \mathcal{Q}_{i}^{*})-\frac{1}{1-k_g}g(\bigcup_{i=1}^{M} \mathcal{P}_{i}).$$
By definition of the objective function $f$,
\begin{align*}
    \frac{\eta}{1-k_g} f(\mathcal{X}_{1:M}) & \geq g(\bigcup_{i=1}^{M} \mathcal{Q}_{i}^{*})-\frac{1}{1-k_g}g(\bigcup_{i=1}^{M} \mathcal{P}_{i}) \\
    &=f(\mathcal{Q}^{*})-\frac{1}{1-k_g}f(\mathcal{X}_{1:M}).
\end{align*}
That is, 
$f(\mathcal{X}_{1:M}) \geq \frac{1-k_g}{1+\eta}f(\mathcal{Q}^{*})$.
\end{proof}

We now use this result in proving our main result.
\subsection{Analysis for Algorithm~\ref{resilient_algorithm}}

\begin{theorem}\label{main_theorem}
Algorithm \ref{resilient_algorithm} returns a set $\mathcal{S}$ such that 
$$f(\mathcal{S} \setminus {\mathcal{A}}^{*}(S)) \geq \frac{\max[1-k_{f},\frac{1}{\alpha+1},\frac{1}{N-\alpha}]}{\frac{1+\eta}{1-k_g}}f^{*}$$
where $\eta, k_g$ are the same as that defined in Theorem \ref{SGA_theorem}; $k_f$ is the curvature of objective function $f$; and ${\mathcal{A}}^{*}(S)$ is the optimal removal set of $\mathcal{S}$; and $f^{*}$ is the optimal solution to RMOP.
\end{theorem}
The omitted proofs can be found in {the appendix}. {Our proof is inspired by that in \cite{tzoumas2017resilient}. The main difference is that the proof in \cite{tzoumas2017resilient} is suitable for set of single elements while our proof deals with set of paths, each of which consists of several elements. Our proof is inspired by one mathematical structure found in \cite{tzoumas2017resilient} on constructing \textit{bait set} and \textit{reward set}. Besides, we need to consider the fact that the single robot submodular OP, which is NP-hard, can only be solved approximately. } 

Now, we describe the three subroutines that can be employed for solving OP. We start with the most general case where the reward function $g$ is submodular and the budget for each robot must be strictly enforced.
\begin{corollary}
If recursive greedy algorithm~\cite{recursivegreedy2005} is used as a subroutine to solve OP and additionally each robot has a predefined terminal vertex, then $\eta$ in Theorem \ref{main_theorem} equals to $\log(OPT)$. Here $OPT$ is the reward collected by the optimal algorithm. The running time of the resulting algorithm is quasi-polynomial since the running time of recursive greedy is quasi-polynomial.
\end{corollary}

Next, consider the variant where $g$ is still submodular, but each robot is allowed to exceed its predefined budget by a bounded amount. 
\begin{corollary}\label{corollary_GCB}
If General Cost-Benefit (GCB) approximation algorithm~\cite{zhang2016submodular} is used as subroutine for OP and we are allowed to relax given budget to $\frac{\psi(n)K_c}{\beta(1+\beta(K_c-1)(1-k_c))}B$, then $\eta$ in Theorem \ref{main_theorem} equals to $2(1-e^{-1})^{-1}$. Here, $\psi(n), \beta, K_c, k_c$ as defined in \cite{zhang2016submodular}. The GCB algorithm runs in polynomial time.
\end{corollary}

Finally, consider the case where $g$ is modular. Here, we get the strongest guarantee with no relaxations to RMOP.
\begin{corollary}
If the reward function $g$ is modular, then using the approximation algorithm for OP \cite{blum2003approximation} yields an $\eta=4$ in Theorem \ref{main_theorem}. The running time of the algorithm \cite{blum2003approximation} is polynomial. 
\end{corollary}

\subsection{Running time}
{
MCTS is an anytime algorithm and can converge to the optimal solution as computational time increases. In applications, the running time is decided by the available computational budget. Next, we will mainly focus on the running time analysis of the proposed Algorithm 1. 
}

{
Let $t_{OP}$ be the time needed to solve a single robot OP and $t_f$ be the time to evaluate the submodular function of a robot path. Suppose that the basic operations like sorting and comparison takes one unit time. Line 2-5 involves solving OP for $N$ times and it takes $O(Nt_{OP})$. Inside the while loop, for line 8, the sorting will take $O(N \rm{lg} \textit{N})$ and evaluation of submodular function will take $O(Nt_f)$. Line 9 and 10 take constant time. SGA (line 11) will take $O((N-\alpha)t_{OP})$. Line 13-19 involves $O((N-\alpha)\alpha)$ comparisons and takes $O(Nt_f)$ to evaluate submodular functions. 
Since sorting and comparing operation is much faster than computing OP and evaluating submodular function, 
the overall running time inside the while loop will be dominated by $O(N(t_f + t_{OP}))$. Suppose that the while loop terminates after $n_w$ loops, which can be upper-bounded as follows. Let $n_i$ be the number of feasible paths for robot $i$ and $n_p = \max_i n_i$. By Corollary 1, the while loop will surely terminate when the reward of each path in $\mathcal{M}$ cannot be increased anymore. The path in $\mathcal{M}$ corresponding to robot $i$ can be improved at most $n_i$ times. As a result, the total number of while loops can be upper-bounded by 
$$n_w \leq \sum_{i=1}^N n_i \leq n_pN.$$
Therefore, the running time for the whole while loop can be upper-bounded by $O(n_p N^2 (t_{OP}+t_f))$. It should be noted that as mentioned in Remark 1, $n_w$ is usually very small in practice. Combining with the running time for line 2 - 5 ($O(Nt_{OP})$), the running time of the algorithm is $O(n_p N^2 (t_{OP}+t_f))$.
}

\section{Numerical Simulations for RMOP without Communication}\label{Simulation}
In this section, we validate the performance of Algorithm~\ref{resilient_algorithm} through numerical simulations. In particular, (1) we compare the performance of our algorithm with two baseline strategies; (2) demonstrate the robustness of the proposed algorithm against attacks that are not necessarily the worst-case ones; and (3) investigate the running time as a function of the size of the input graph and the number of robots.
All experiments were performed on a Windows 64-bit laptop with 16 GB RAM and an 8-core Intel i5-8250U 1.6GHz CPU using Python 3.7.  

\subsection{Simulation Setup}
\begin{figure}[ht]
    \subfloat[SGA paths]{
    \includegraphics[width=0.24\textwidth]{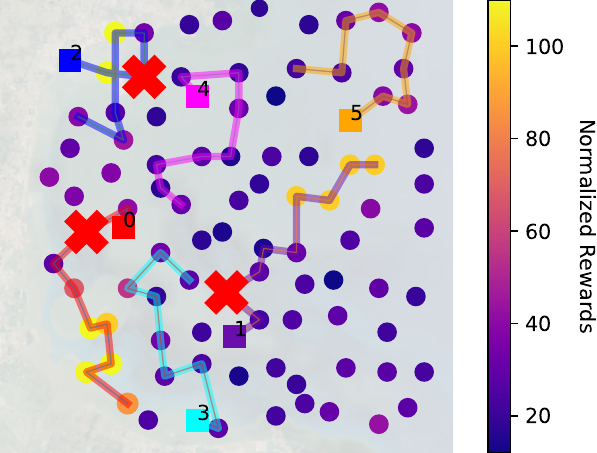}
    \label{Sim_SGA}
    }
    \subfloat[RMOP paths]{
    \includegraphics[width=0.24\textwidth]{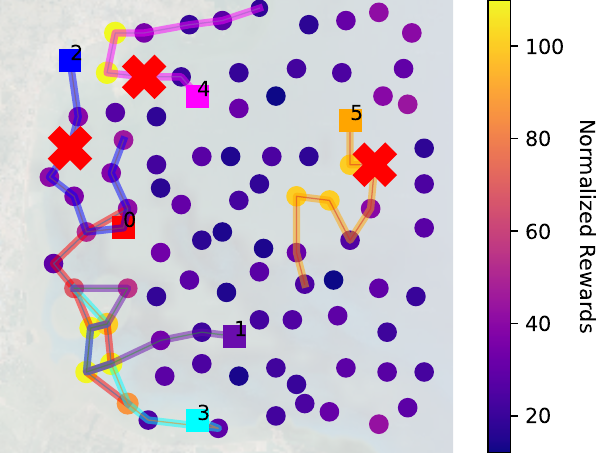}
    \label{Sim_RMOP}
    }
	\caption{Case study of monitoring macroalgal blooms using $N=6$ robots assuming $\alpha=3$ failures. Colored dots indicate locations to be monitored along with their importance (i.e., rewards). Red crosses indicate worst-case attacks found using brute force. Paths returned by the proposed algorithm manages to cover one of the three important areas (lower left corner) while SGA loses all three. The background map is part of the Yellow Sea, where green tides prevail every summer since 2007.
	}
	\label{scenario}
\end{figure}
\paragraph*{Application case study} We use the application of monitoring a marine environment for mapping oil leaks, macroalgal blooms, or pH values. Specifically, as explained in \cite{xing2019monitoring}, such tasks usually calls for collaboration of multiple sensors including satellites, which can provide coarse prior information on the concentration of the phenomenon of interest, and mobile robotic sensors, which can use the prior information for targeted data collection. Using this as motivation, we consider a scenario where prior information from satellites (for example) can be used to define an importance map over the environment to be monitoring. Fig. \ref{scenario} shows the setup which consists of 96 vertices placed in the environment. The color of the vertex reflects the importance of that vertex which gives the reward associated with visiting that vertex. Here, the single robot function, $g(\mathcal{P}_i)$, is a modular reward. The cost along the edges is the Euclidean distance between the vertices. Assuming unit speed of travel, the cost of a path reflects the travel time of the robot. In all the instances, each robot is given a budget $B=60$ units. 

\paragraph*{Baseline algorithms} Since we introduce RMOP in this paper, there is no other efficient algorithm to directly compare the performance with. One option is to compute the optimal solution (using, for example, brute-force enumeration) which quickly becomes intractable. Instead, we choose two approximation algorithms for MOP, the non-adversarial version, as baselines. The first one uses the sequential greedy assignment for all $N$ robots (we refer to it as SGA) where the path for robot $i$ is based on the paths computed for robots $1$ through $i-1$. The second baseline is the naive greedy algorithm where each robot naively (without considering the travel cost) and greedily (without considering other robots) maximize their rewards (we refer to it as NG). 

For both SGA and the proposed algorithm, we use the GCB algorithm solving OP due to its efficiency and ease of implementation. Specifically, we implement GCB using details provided in \cite{zhang2016submodular}. When running GCB, we simply set the relaxed budget itself to be $B$.

\paragraph*{Attack models} Our algorithm is designed to give performance guarantees against worst-case attacks. However, in practice, we would like for any algorithm to be robust to not just the worst-case attacks but also other attacks. Therefore, we evaluate two other types of attacks besides worst-case attacks. The details are provided in the next subsection.

\subsection{Results}
Fig.~\ref{scenario} shows a qualitative example comparing our proposed algorithm and SGA with $N=6$ and $\alpha=3$. Not surprisingly, the six paths found by SGA do not have any overlap but the ones found by our algorithm do. As a result, the worst-case attack takes away all three robots covering the important regions in SGA, whereas one of the three regions is still covered with our algorithm. The worst-case attacks were computed using brute force.

Next, we present quantitative results. In all the following figures, the error bar shows the variance of 20 trials where the starting robot positions are randomly chosen.

\begin{figure}[ht]
    \subfloat[\label{reward_after_worst_attack}]{
    \centerline{
    \includegraphics[width=0.3\textwidth]{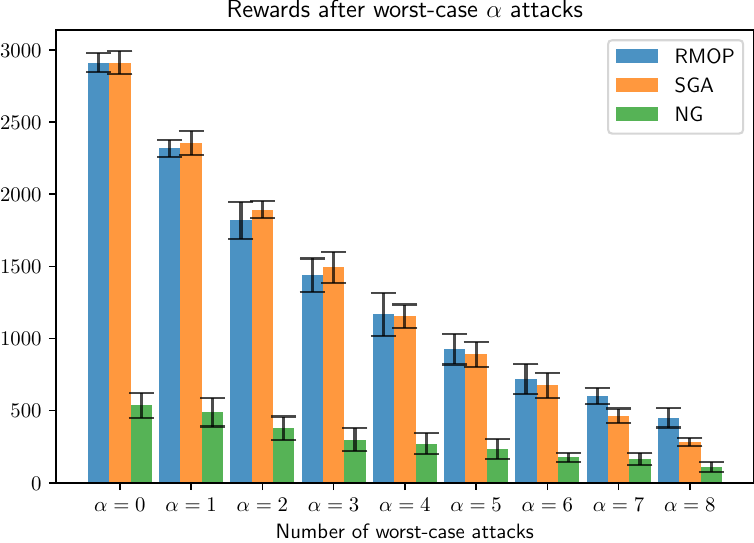}
    }
    }\\
    \subfloat[\label{reward_after_random_attack}]{
    \centerline{
    \includegraphics[width=0.3\textwidth]{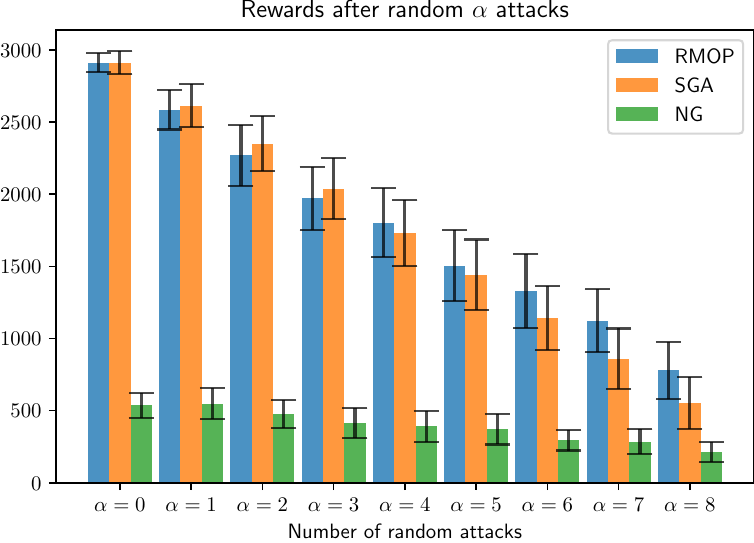}
    }
    }
	\caption{(a) Rewards after worst-case attack with increasing $\alpha$ and $N=10$. (b) Rewards after random $\alpha$ attacks and $N=10$. } 
	\label{reward_left}
\end{figure} 
Fig.~\ref{reward_after_worst_attack} shows the comparison between our algorithm for RMOP with SGA and NG as $\alpha$ increases with $N=10$. The bars show the rewards collected by the robots after attacks. Our algorithm returns paths that are slightly worse than SGA when $\alpha$ is small. This is not surprising since our algorithm will have overlapping paths whereas SGA will not. NG is the other extreme since each robot plans for itself which can lead to a high degree of overlap. {Though our algorithm has only comparable performance to SGA when $\alpha$ is relatively small,  the reward gap is small.  As $\alpha$ increases, our algorithm gradually significantly outperforms SGA. For example, when $\alpha=8$ our algorithm yields a reward of $451$ whereas SGA only yields $283$ on average. As a result,  in general,  the overall performance of the Algorithm 1 can be trusted especially for the large $\alpha$.  Moreover, in practice, since we cannot decide the threshold above which Algorithm 1 significantly outperforms SGA, we can run SGA and Algorithm 1 in parallel for offline planning if the planning budget permits and select the one that returns higher rewards.  In this way, we can preserve the same theoretical guarantee and have possibly better practical performance. }

Next, we evaluate the performance of our algorithm when the attack model does not match the worst-case one assumed during planning. The goal is to verify the robustness of the algorithm to other attack models. 
Fig.~\ref{reward_after_random_attack} shows the comparison between our algorithm and the two baselines as $\alpha$ varies when the attacked robots are randomly chosen. Our algorithm still plans to assume worst-case attacks. We observe the same trend with random attacks as with the worst-case ones --- as $\alpha$ increases, our algorithm outperforms SGA.

\begin{figure}
\centering
\includegraphics[width=0.35\textwidth]{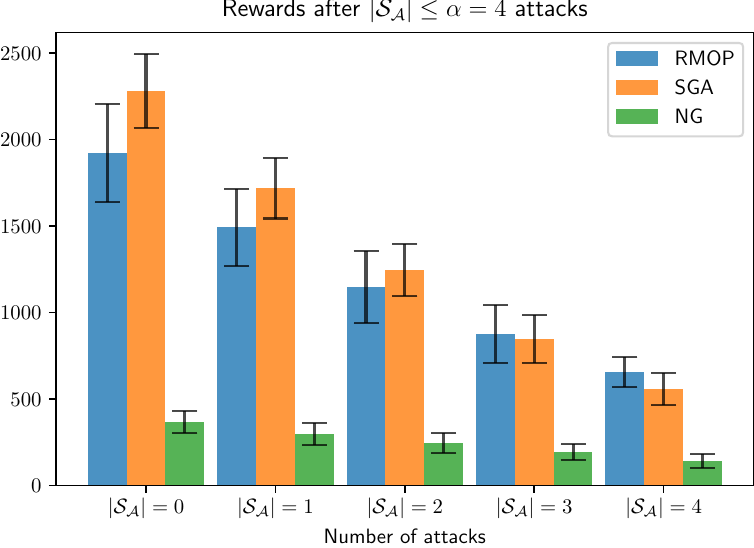}
\caption{Rewards after worst-case attacks of increasing sizes, $|\mathcal{S_A}| \leq \alpha$. Here the planner uses $N=7$ and $\alpha=4$.
\label{fig:salpha}
}
\end{figure}

\begin{figure}[ht]
    \subfloat[7 robots with $\alpha=4$]{
    \includegraphics[width=0.23\textwidth]{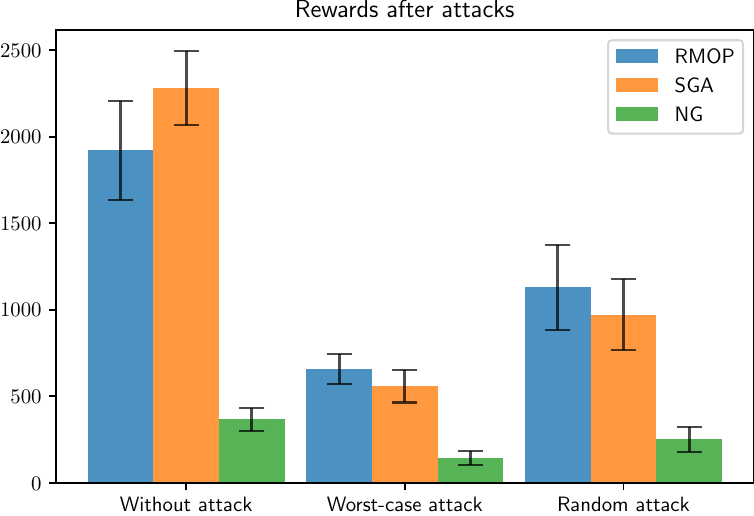}
    \label{reward_7_4}
    }
    \subfloat[8 robots with $\alpha=4$]{
    \includegraphics[width=0.23\textwidth]{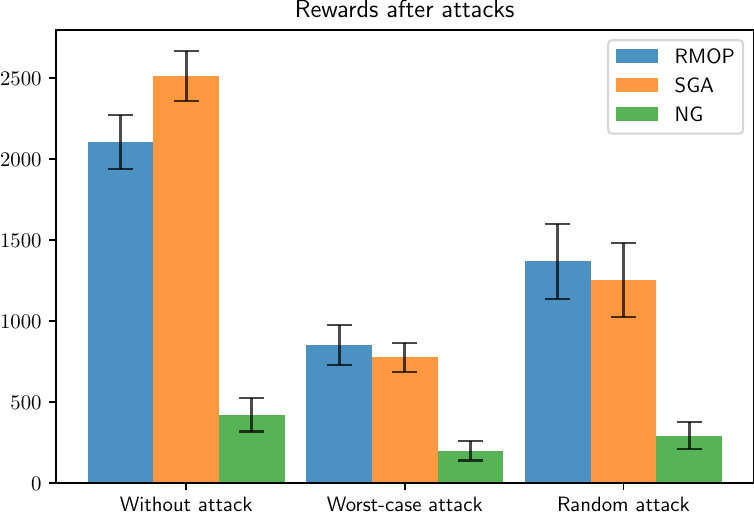}
    \label{reward_8_4}
    }\\
    \subfloat[9 robots with $\alpha=5$]{
    \includegraphics[width=0.23\textwidth]{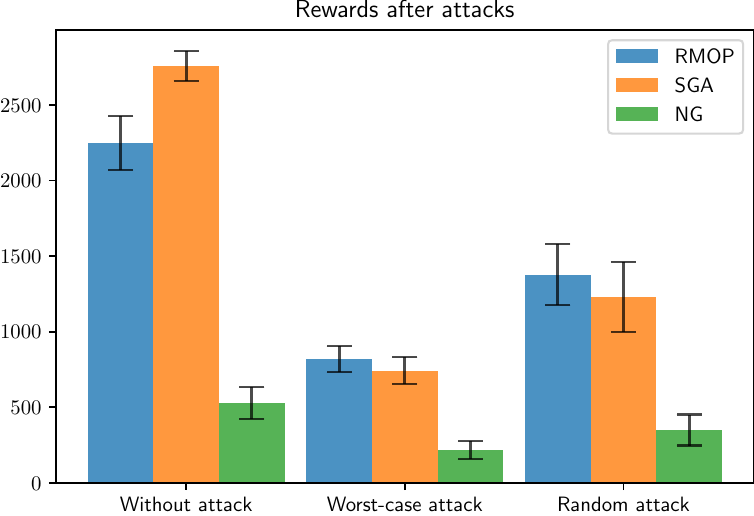}
    \label{reward_9_5}
    }
    \subfloat[10 robots with $\alpha=6$]{
    \includegraphics[width=0.23\textwidth]{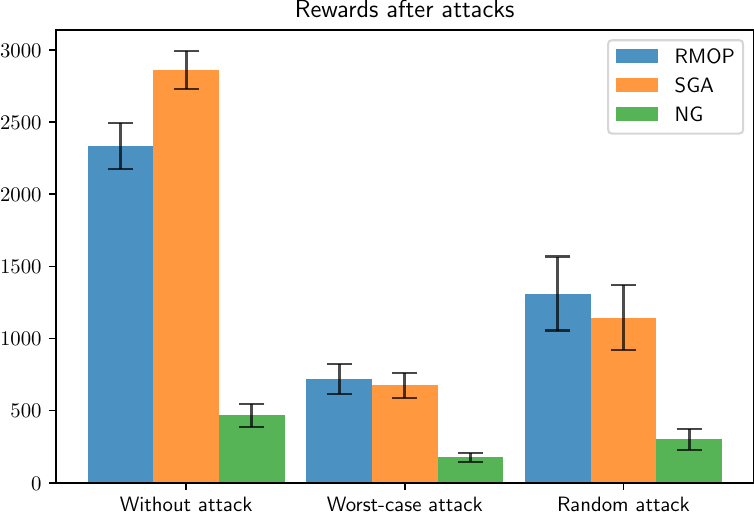}
    \label{reward_10_6}
    }
	\caption{Rewards after: (1) no attacks; (2) $\alpha$ out of $N$ robots under worst-case attack; (3) $\alpha$ out of $N$ robots under random attack.
	} 
	\label{reward_after_attack}
\end{figure} 
Fig. \ref{fig:salpha} shows results for the case where we construct paths assuming $\alpha=4$ robots will be attacked but in practice only $|\mathcal{S_A}|\leq \alpha$ robots suffer from worst-case attacks. SGA performs better than our algorithm when the number of robots actually attacked $|\mathcal{S_A}|$ is far from the designed value of $\alpha$. As the actual number of robots attacked increases and $|\mathcal{S_A}|$ approaches $\alpha$, our algorithm outperforms SGA. {This suggests that we need to have an accurate estimate of the $\alpha$ before applying our algorithm, which is one limitation of this paper. However, in many robotic applications, it's possible to estimate $\alpha$ using historical data. Take the information gathering in the marine environment for example. We can use the number of robots that survived in the previous executions of the mission to estimate $\alpha$.}

Fig.~\ref{reward_after_attack} shows the evaluation  
when there are (1) no attacks; (2) worst-case attacks; and (3) random attacks for four configurations of $N$ and $\alpha$. In all three cases, our algorithm still plans the paths assuming worst-case attacks for the given value of $\alpha$. When there are no attacks (first set of bars in each subfigure), SGA outperforms our algorithm as observed in previous charts. When worst-case attacks do happen (middle set of bars), the average rewards collected by the unattacked robots employing our algorithm are better than that of SGA. This is also the case when the $\alpha$ attacked robots are chosen randomly (third set of bars). This trend holds for various values of $N$ and $\alpha$ as shown.

\begin{figure}[ht]
    \subfloat[Running time w.r.t. vertices \label{running_time_vertices}]{
       \includegraphics[width=0.24\textwidth]{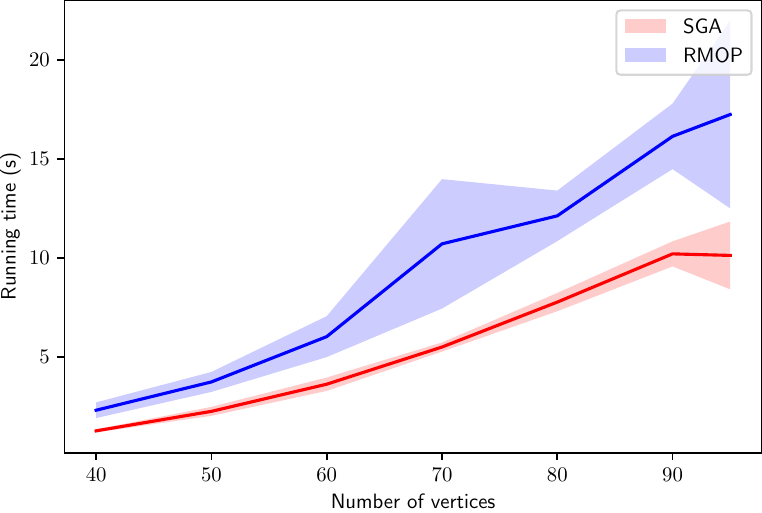}
     }
    \subfloat[Running time w.r.t. robots \label{running_time_robots}]{
       \includegraphics[width=0.24\textwidth]{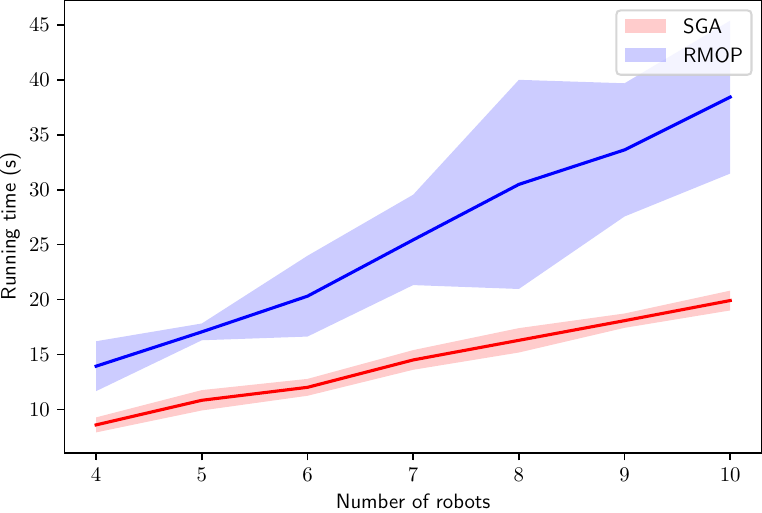}
     }
	\caption{Running time of Algorithm \ref{resilient_algorithm} and SGA}
	\label{running_time}
\end{figure} 
The running time comparisons between our algorithm and SGA are shown in Fig. \ref{running_time}. We vary the number of robots as well as the size of the graph. Our algorithm takes longer than SGA which is expected since it uses SGA as a subroutine. Nevertheless, we observe similar trends in the runtime.

\paragraph*{Discussion of Results} The results show the proposed algorithm works in practice as intended. As the number of attacked robots increases (either $\mathcal{S}_\mathcal{A}$ or $\alpha$), it outperforms SGA. Furthermore, we observe that the margin between our algorithm and SGA increases as the number of attacked robots increases. Even when our algorithm finds worse paths than SGA, they are still comparable to SGA and are significantly better than NG. We also observe that our algorithm is robust to the actual attack models --- even if the attacks are not the worst-case ones, we see similar trends. 

\section{Numerical Simulations for RMOP with Communication}\label{sim:RMOP with communication}
In this section, we validate the performance of Algorithm \ref{MCTS_algorithm} for the online RMOP. In particular, we present a case study on information gathering in a tunnel and compare the performance of the team when robots and attackers adopt different strategies.

\subsection{Simulation Setup}
\paragraph*{Application case study}We use the application of information gathering in a tunnel in which we assume that some coarse prior information on the rewards of some locations are known and a team of robots is sent to gather detailed information but at most $\alpha$ attackers may attack them at any locations. We also assume that communication, though maybe degraded, is available. We use a tunnel map from Defense Advanced Research Projects Agency (DARPA) Subterranean Challenge and use image skeletonization algorithm and corner detection algorithm \cite{van2014scikit} to identify several points as locations of interest. Fig. \ref{sim:tunnel} shows the setup which consists of 69 vertices placed in the environment. The color and the size of the node reflect the importance of that vertex which gives the reward associated with visiting that vertex. The corresponding graph abstraction is shown in Fig. \ref{sim:tunnel_graph} in which we assume that it takes one unit of time to transit between two adjacent nodes. Even though adding some virtual nodes on the long edges will make this assumption better justified, for simplicity, we ignore these virtual nodes. Here, the single robot function, $g(\mathcal{P}_i)$, is a modular reward. The cost along the edge is the distance between two adjacent nodes which is defined in the image coordinate as the shortest distance in the skeleton image (739 $\times$ 520 pixels). In this case study, there are four robots in the environment and $\alpha=2$ attackers and each robot has a budget of $B=500$ units. 

 We also test the performance of our strategy compared to other strategies in the randomly generated graphs. We generate four $15 \times 15$ grid graphs, whose edges are of unit length, to represent the environments. To account for the sparsity of the tunnel environment, half of the edges are randomly removed in each graph but the graph remains connected. For each graph, the reward of each node is generated by sampling a number from one exponential distribution. We use different rate parameters for different graphs. It should be noted that our algorithm doesn't depend on the particular distribution of rewards. We choose the exponential distribution just for its non-negative support. In all instances, there are four robots in the environment and $\alpha=2$ attackers and each robot has a budget of $B=8$ units. 

\begin{figure}[ht]
    \subfloat[Tunnel map]{
    \includegraphics[width=0.28\textwidth]{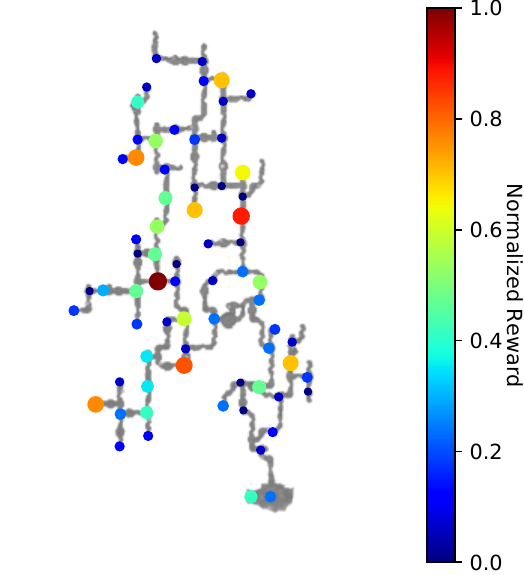}
    \label{sim:tunnel}
    }
    \subfloat[Tunnel graph]{
    \includegraphics[width=0.16\textwidth]{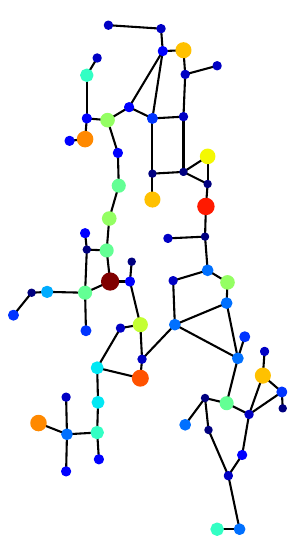}
    \label{sim:tunnel_graph}
    }
	\caption{A tunnel map for a case study of information gathering. Colored dots with various sizes indicate locations to be visited along with their importance i.e., rewards. (a) A tunnel map with 69 locations of interest. (b) The graph abstraction of the tunnel map. The tunnel map is from Defense Advanced  Research  Projects  Agency (DARPA) Subterranean Challenge.
	}
	\label{fig:tunnel}
\end{figure}

\paragraph*{Strategies} We consider two strategies for robots including MCTS with adversaries (Algorithm \ref{MCTS_algorithm}, two-player search, we refer it as MA) and naive MCTS (don't consider failures/attacks, one-player search, we refer it as M) and two strategies for attackers including MCTS with adversaries (Similar to Algorithm \ref{MCTS_algorithm}, two-player search, we refer it as MA, and RandomMove (randomly choose an available action each time, we refer it as R). Problem \ref{RMOP_with_communication} is simulated as a two-player game in which at each step attackers first use their strategy to choose one available action and then robots choose one action. Such a process continues until all robots run out of budget. 

\subsection{Results}

\begin{figure*}[ht]
    \subfloat[robot 0]{
    \includegraphics[width=0.23\textwidth]{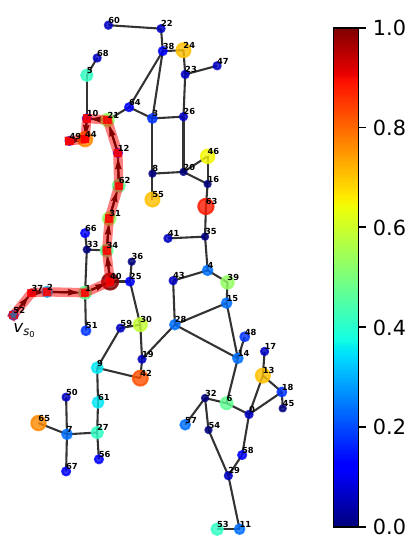}
    \label{sim:robot0_path}
    }
    \subfloat[robot 1]{
    \includegraphics[width=0.23\textwidth]{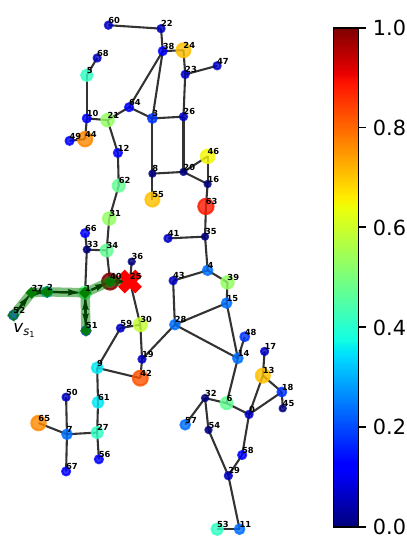}
    \label{sim:robot1_path}
    }
    \subfloat[robot 2]{
    \includegraphics[width=0.23\textwidth]{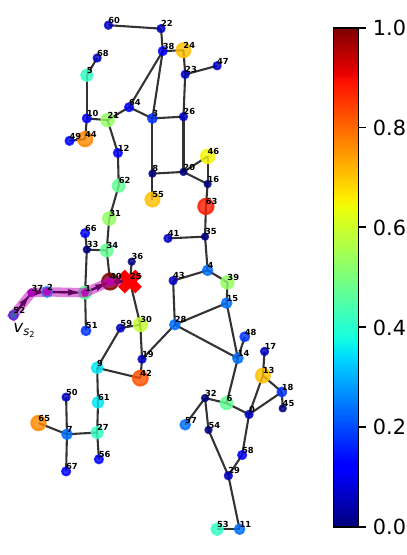}
    \label{sim:robot2_path}
    }
    \subfloat[robot 3]{
    \includegraphics[width=0.23\textwidth]{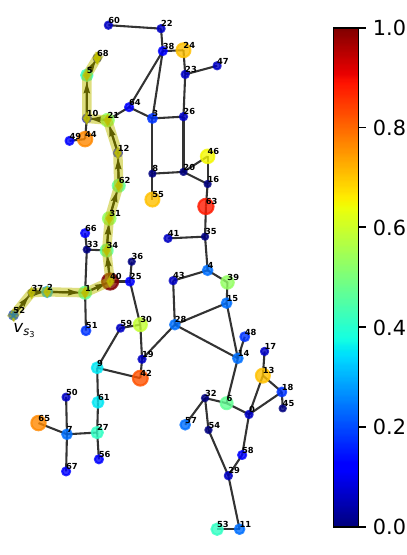}
    \label{sim:robot3_path}
    }
	\caption{A case study of information gathering in a tunnel with $N=4$ robots assuming $\alpha=2$ failures/attacks. Colored dots indicate locations to be visited along with their importance (i.e., rewards). Red across indicates the attacks found when robots and attackers play the two-player sequential game and both use MCTS with adversaries. $v_{s_i}$ represents the starting position of the robot $i$. Attackers launch the first attack when they observe that robot 0 and 3 reach node 34 and robot 2 reaches node 25 at the fifth step and launch another attack later when robot 1 reaches node 25 at the seventh step. (a) robot 0 follows the path $52 \rightarrow 37 \rightarrow 2 \rightarrow 1 \rightarrow 40 \rightarrow 34 \rightarrow 31 \rightarrow 62 \rightarrow 12 \rightarrow 21 \rightarrow 10 \rightarrow 44 \rightarrow 49$. (b) robot 1 follows the path $52 \rightarrow 37 \rightarrow 2 \rightarrow 1 \rightarrow 51 \rightarrow 1 \rightarrow 40 \rightarrow 25$ and is attacked after two steps. (c) robot 2 follows the path $52 \rightarrow 37 \rightarrow 2 \rightarrow 1 \rightarrow 40 \rightarrow 25$ and is attacked after three steps. (d) robot 3 follows the path $52 \rightarrow 37 \rightarrow 2 \rightarrow 1 \rightarrow 40 \rightarrow 34 \rightarrow 31 \rightarrow 62 \rightarrow 12 \rightarrow 21 \rightarrow 10 \rightarrow 5 \rightarrow 68$. 
	} 
	\label{sim:MCTS_paths}
\end{figure*} 
Fig. \ref{fig:tunnel} shows the tunnel map from DARPA subterranean challenge. Fig. \ref{sim:tunnel} is the original tunnel map with 69 locations of interest and Fig. \ref{sim:tunnel_graph} is the corresponding abstract graph representation of the tunnel map. Fig. \ref{sim:MCTS_paths} shows a qualitative example of how robots and attackers behave in a two-player sequential game when both of them use MCTS with adversaries. In each step, attackers will first grow the search tree based on what they observed so far and choose an action. Then, robots will grow the search tree and choose one action. Such a process continues until the budgets of robots are used up. As a result, robot 0 follows the path $52 \rightarrow 37 \rightarrow 2 \rightarrow 1 \rightarrow 40 \rightarrow 34 \rightarrow 31 \rightarrow 62 \rightarrow 12 \rightarrow 21 \rightarrow 10 \rightarrow 44 \rightarrow 49$; robot 1 follows a path $52 \rightarrow 37 \rightarrow 2 \rightarrow 1 \rightarrow 51 \rightarrow 1 \rightarrow 40 \rightarrow 25$ and is attacked at node 25; robot 2 follows a path $52 \rightarrow 37 \rightarrow 2 \rightarrow 1 \rightarrow 40 \rightarrow 25$ and is attacked at node 25; robot 3 follows a path $52 \rightarrow 37 \rightarrow 2 \rightarrow 1 \rightarrow 40 \rightarrow 34 \rightarrow 31 \rightarrow 62 \rightarrow 12 \rightarrow 21 \rightarrow 10 \rightarrow 5 \rightarrow 68$. As shown in Fig. \ref{sim:MCTS_paths}, attackers launch the first attack when they observe that robots 0 and 3 reach node 34 and robot 2 reaches node 25 because if they don't attack at that moment robot 2 will move downward to collect more rewards. Attackers launch another attack later when robot 2 reaches node 25 and may collect more rewards from the nodes below.

Fig. \ref{sim:different strategies} shows the results when robots and attackers use different strategies in four graphs. Robots are randomly initialized in different vertices and for each initialization, the two-player game is conducted 20 times. The collected reward of the team is the sum of the rewards of nodes visited. As shown in Fig. \ref{sim:different strategies}, when attackers use MCTS with adversaries (first two bars blue and orange), robots can collect more reward on average if they also use the MCTS with adversaries (orange) compared to the case where they use a naive MCTS without considering attacks. If attackers use a random strategy (last two bars green and red), the MCTS with adversaries strategy (red) is also on average better than a naive MCTS without considering attacks (green). Moreover, robots can collect more rewards on average if attackers just select an action randomly (green and red compared to blue and orange).  
\begin{figure}[ht]
    \subfloat[]{
    \includegraphics[width=0.23\textwidth]{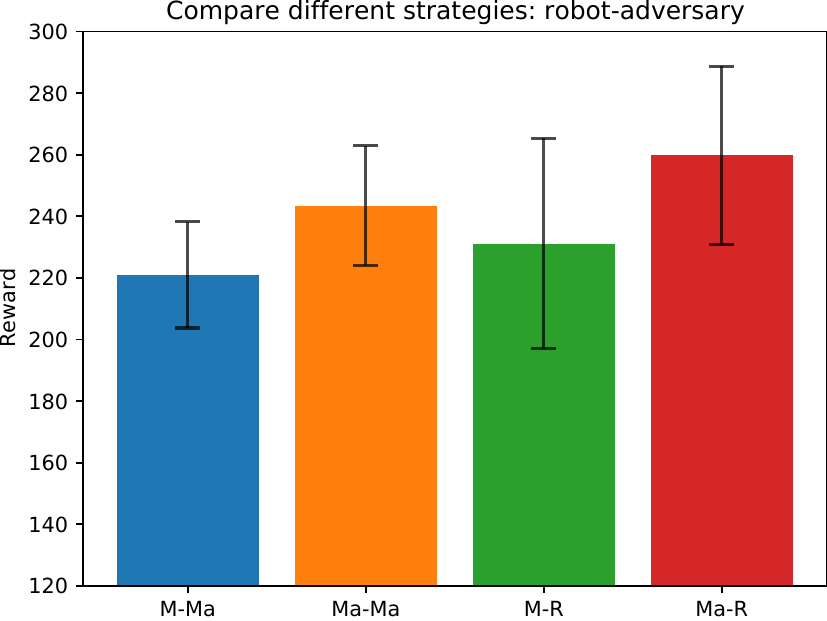}
    \label{sim:graph1}
    }
    \subfloat[]{
    \includegraphics[width=0.23\textwidth]{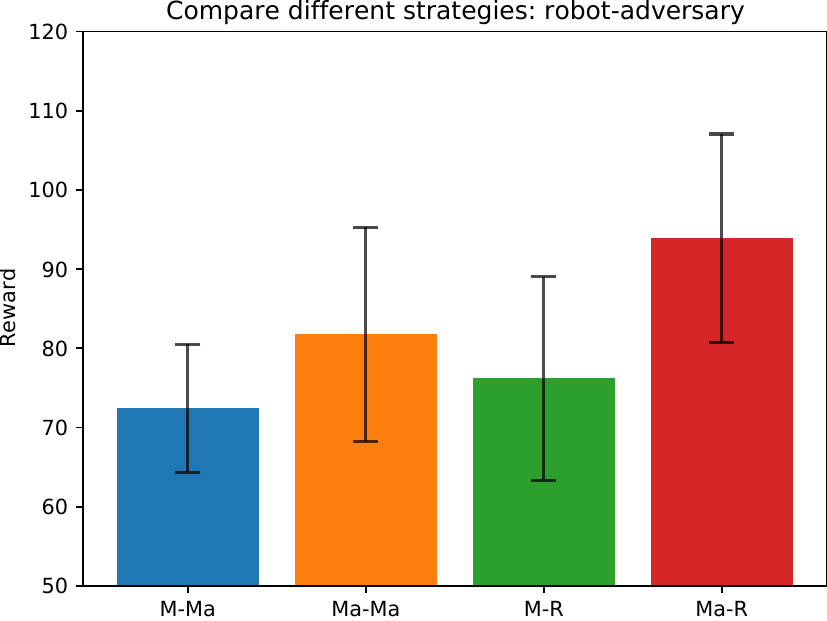}
    \label{sim:graph2}
    }\\
    \subfloat[]{
    \includegraphics[width=0.23\textwidth]{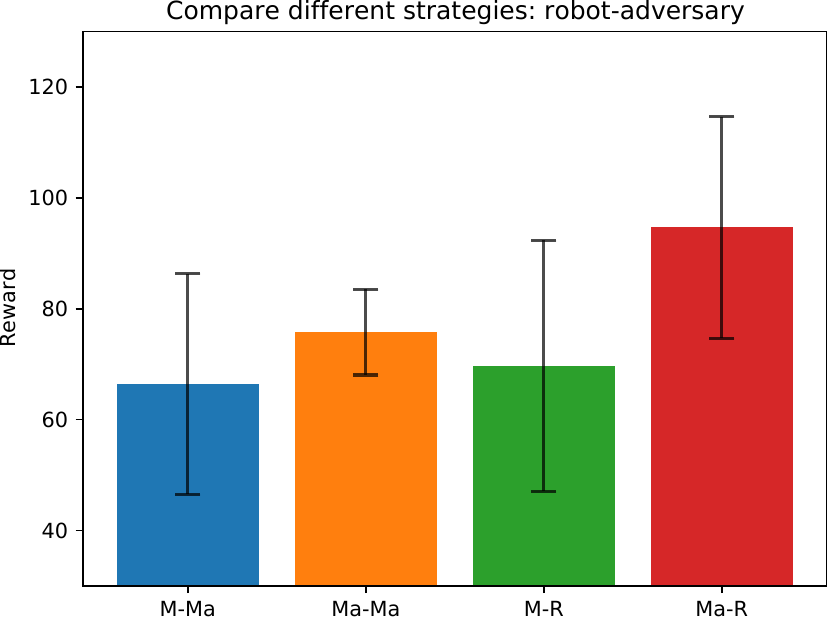}
    \label{sim:graph3}
    }
    \subfloat[]{
    \includegraphics[width=0.23\textwidth]{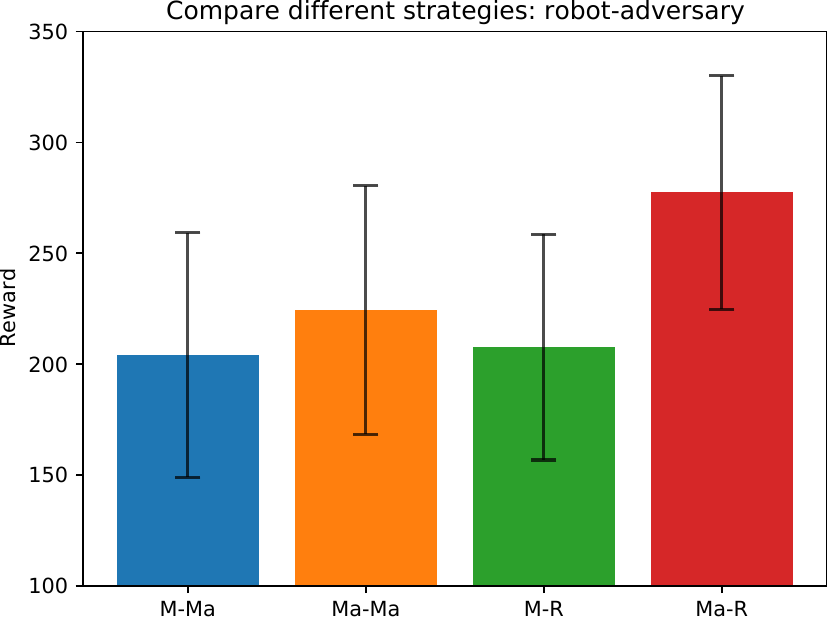}
    \label{sim:graph4}
    }
	\caption{Collected rewards for different strategies in four different maps: (1) M-MA: robots use one-player MCTS (M) and attackers use MCTS with adversaries (MA); (2) MA-MA: both robots and attackers use MCTS with adversaries; (3) M-R: robots use one-player MCTS (M) and attackers attacks randomly (R); (4) Ma-R: robots use MCTS with adversaries and attackers attacks randomly (R). 
	} 
	\label{sim:different strategies}
\end{figure}

\section{Conclusion}
We introduced a new problem, termed Robust Multiple-Path Orienteering Problem, in which we seek to construct a set of paths for robots such that even if a subset of robots fails, the rest of the team still performs well. We consider two types of RMOP. In the offline RMOP in which robots cannot communicate with each other or the base station during the execution of tasks, we provided a general approximation framework for the offline RMOP, which builds on bounded approximation algorithms for OP and the sequential greedy assignment framework. We showed three variants of the general algorithm that use three different subroutines for OP and still yield a bounded approximation for RMOP. In addition to theoretical results, we presented empirical results showing that our algorithm is robust to attacks other than the worst-case ones. We also compare our performance with baseline algorithms and show that our algorithm yields better performance as more and more robots are attacked.
In the online version, RMOP is modeled as a two-player sequential game and solved adaptively in a receding horizon fashion based on Monte Carlo Tree Search (MCTS). Simulation results show that MCTS with adversaries performs better on average than the MCTS without considering attacks/failures.

\bibliographystyle{IEEEtran}
\bibliography{IEEEabrv,resilient_orienteering}

\appendix
\section*{Supplementary Materials}

\subsection{Proof of Proposition~\ref{no_overlap_assumption}}
\begin{proof}
\begin{figure}[htbp]
\centerline{\includegraphics{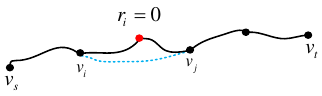}}
\caption{A feasible path that contains vertex (red) with zero rewards. A new path can be constructed by connecting the vertices before ($v_i$) and after $v_j$ the zero-reward vertex (blue dotted line). The new path has the same reward and is feasible since we have a metric graph. }
\label{path_indep}
\end{figure}
Suppose that $\exists  i, j  \text{ such that } \mathcal{Q}^{*}_{i} \cap \mathcal{Q}^{*}_{j}=\{v\}$. According to the definition of objective function $f$, vertex $v$ will contribute to $f$ only once even if it appears in multiple paths. Therefore, if we keep vertex $v$ in $ \mathcal{Q}^{*}_{i}$, it becomes a `zero-reward' vertex for $\mathcal{Q}^{*}_{j}$. Then we can directly connect the vertices before and after that `zero-reward' vertex to obtain  $\mathcal{Q}^{*}_{j}$ without changing the overall performance of the team. The new path $\mathcal{Q}^{*}_{j}$ is still feasible since we assume the graph is metric. An illustrative example is shown in Fig. \ref{path_indep}. {As for the case where there are multiple intersect vertices between two paths, we can use the above argument for each vertex one by one. }
\end{proof}

\subsection{Proof of Theorem \ref{main_theorem}}
\begin{proof}
Our proof relies on the following three inequalities:
\begin{equation}\label{s2_ineq1}
    f(\mathcal{S} \setminus {\mathcal{A}}^{*}(S)) \geq (1-k_f)f(\mathcal{S}_2)
\end{equation}
\begin{equation}\label{s2_ineq2}
    f(\mathcal{S} \setminus {\mathcal{A}}^{*}(S)) \geq \frac{1}{\alpha+1}f(\mathcal{S}_2)
\end{equation}
\begin{equation}\label{s2_ineq3}
    f(\mathcal{S} \setminus {\mathcal{A}}^{*}(S)) \geq \frac{1}{N-\alpha}f(\mathcal{S}_2)
\end{equation}
The proof for inequality \eqref{s2_ineq1} is similar to the proof (Eq. 16-20) in \cite{tzoumas2017resilient} when combined with the invariant maintained by the \textbf{while} loop in Algorithm \ref{resilient_algorithm}. Likewise, the proof for inequality \eqref{s2_ineq2} resembles that given in \cite{tzoumas2017resilient} (from Eq. 21 to Eq. 25) and inequality \eqref{s2_ineq3} can be proved in the same fashion as that in \cite{tzoumas2018resilientnon} (Eq. 57 to Eq. 58). For completeness, we give the proofs in the supplementary document.

From Theorem \ref{SGA_theorem} we have,
\begin{equation}\label{S2_S2star}
    f(\mathcal{S}_2) \geq \frac{1}{\frac{1+\eta}{1-k_g}}f(\mathcal{Q}^{*})
\end{equation}
where $\mathcal{Q}^{*}$ is the optimal solution to the multi-path orienteering problem for robots $\mathcal{R}\setminus\mathcal{R}(\mathcal{S}_1)$.
According to Lemma 9 in \cite{tzoumas2018resilientnon}, we have another inequality:
\begin{equation}\label{Q_star2global_optimal}
    f(\mathcal{Q}^{*}) \geq f^{*}=f(\mathcal{S}^{*}\setminus\mathcal{A}^{*}(\mathcal{S}^{*}))
\end{equation}
where $f^{*}$ is the optimal solution to RMOP. For completeness, we give the proof in the supplementary document.
{Combining} inequalities \eqref{s2_ineq1} -- \eqref{Q_star2global_optimal}, we get the statement for Theorem \ref{main_theorem}. 
\end{proof}

\subsection{Proof of Inequality in Equation (\ref{s2_ineq1})}
We will prove the following,
\begin{equation}\label{ineq1}
    f(\mathcal{S} \setminus {\mathcal{A}}^{*}(S)) \geq (1-k_f)f(\mathcal{S}_2)
\end{equation}
where $\mathcal{S}=\mathcal{S}_1 \cup \mathcal{S}_2$ is the solution returned by our algorithm; ${\mathcal{A}}^{*}(S)$ is the optimal removal of $\mathcal{S}$; $k_f$ is the curvature of function $f$.
Next, we prove inequality \ref{ineq1} and we will use Lemma 1 from \cite{tzoumas2017resilient} without proof. Proof here is essentially the same as that in \cite{tzoumas2017resilient} but with different notations for better understanding.
\begin{lemma}\label{A2a}
Consider a finite ground set $\mathcal{V}$ and a monotone set function $f:2^{\mathcal{V}}\rightarrow\mathbb{R}$ such that $f$ is a non-negative and $f(\emptyset)=0$. For any set $\mathcal{A} \subseteq \mathcal{V}$, 
\begin{equation}
    f(\mathcal{A}) \geq (1-k_f)\sum_{a \in \mathcal{A}}f(a)
\end{equation}
\end{lemma}

\begin{figure}[ht]
	\centerline{\includegraphics[scale=0.7]{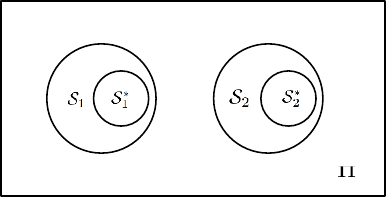}}
	\caption{Venn diagram, where $\mathcal{S}_1, \mathcal{S}_2, \mathcal{S}_1^{*}, \mathcal{S}_2^{*}$ are defined as follows: Per run of proposed algorithm for RMOP, $\mathcal{S}_1$ and $\mathcal{S}_2$ are intermediate results such that $\mathcal{S}=\mathcal{S}_1 \cup \mathcal{S}_2$, and $\mathcal{S}_1 \cap \mathcal{S}_2 = \emptyset$. Let $\mathcal{S}_{\mathcal{A}}^{*}(\mathcal{S})$ be the optimal removal from $\mathcal{S}$. Then $\mathcal{S}_1^{*}, \mathcal{S}_2^{*}$ are defined such that $\mathcal{S}_1^{*}={\mathcal{A}}^{*}(\mathcal{S}) \cap \mathcal{S}_1$ and $\mathcal{S}_2^{*}={\mathcal{A}}^{*}(\mathcal{S}) \cap \mathcal{S}_2$. By definition, 
	$\mathcal{S}_1^{*} \cap \mathcal{S}_2^{*}=\emptyset$ and $\mathcal{S}_{\mathcal{A}}^{*}(\mathcal{S})=\mathcal{S}_1^{*} \cup \mathcal{S}_2^{*}$.
	} 
	\label{Venn diagram}
\end{figure}
Let $\mathcal{S}_1^{+}=\mathcal{S}_1 \setminus \mathcal{S}_1^{*}$ and $\mathcal{S}_2^{+}=\mathcal{S}_2 \setminus \mathcal{S}_2^{*}$
\begin{align}
    f(\mathcal{S}\setminus{\mathcal{A}}^{*}(\mathcal{S})) &= f(\mathcal{S}_1^{+} \cup \mathcal{S}_2^{+}) \label{V_1}\\
    & \geq (1-k_f) \sum_{s \in \mathcal{S}_1^{+} \cup \mathcal{S}_2^{+}}f(s) \label{V_2}\\
    & \geq (1-k_f)(\sum_{s \in \mathcal{S}_2 \setminus \mathcal{S}_2^{+}}f(s) +
     \sum_{s \in \mathcal{S}_2^{+}}f(s)) \label{V_3}\\
    & \geq  (1-k_f)(f(\mathcal{S}_2 \setminus \mathcal{S}_2^{+})+f(\mathcal{S}_2^{+})) \label{V_4}\\
    & \geq (1-k_f)((\mathcal{S}_2 \setminus \mathcal{S}_2^{+}) \cup \mathcal{S}_2^{+}) \label{V_5}\\
    & = (1-k_f)f(\mathcal{S}_2) \label{V_6}
\end{align}
where \eqref{V_1} holds by definition; \eqref{V_1} to \eqref{V_2} holds due to Lemma \ref{A2a}; \eqref{V_3} follows from \eqref{V_2} since all paths $s \in \mathcal{S}_1^{+}$ and all paths $s^{\prime} \in \mathcal{S}_2 \setminus \mathcal{S}_2^{+}$, the inequality $f(s) \geq f(s^{\prime})$ holds, i.e. paths in $\mathcal{S}_1$ have more rewards compared to that in $\mathcal{S}_2$ (note that by definitions of sets $\mathcal{S}_1^{+}$ and $\mathcal{S}_2^{+}$ it is $|\mathcal{S}_1^{+}|=|\mathcal{S}_2^{*}|=|\mathcal{S}_2 \setminus \mathcal{S}_2^{+}|$, i.e. the number of non-removed paths in $\mathcal{S}_1$ is equal to the number of removed paths in $\mathcal{S}_2$); from \eqref{V_3} to \eqref{V_4}, it is due to submodularity; and \eqref{V_6} follows from \eqref{V_5} by definition. 

\subsection{Proof of Inequality in Equation (\ref{s2_ineq2})}
We will prove the following,
\begin{equation}\label{ineq2}
    f(\mathcal{S} \setminus {\mathcal{A}}^{*}(S)) \geq \frac{1}{\alpha+1}f(\mathcal{S}_2)
\end{equation}
We will use Lemma 2 from \cite{tzoumas2017resilient} without proof. Proof here is essentially the same as that in \cite{tzoumas2017resilient} but with notations consistent with our paper for better understanding.
\begin{lemma}\label{a_geq_b}
Consider any finite ground set $\mathcal{V}$, a monotone submodular function $f:2^{\mathcal{V}}\rightarrow\mathbb{R}$ such that $f$ is a non-negative and $f(\emptyset)=0$. Consider two non-empty sets $\mathcal{Y}, \mathcal{P} \subseteq \mathcal{V}$ such that for all elements $y \in \mathcal{Y}$ and all elements $p \in \mathcal{P}$ it is $f(y) \geq f(p)$. Then, 
\begin{equation}
    f_{\mathcal{Y}}(\mathcal{P}) \leq |\mathcal{P}|f(\mathcal{Y})
\end{equation}
\end{lemma}
First we introduce one notation:
\begin{equation}
    \xi = \frac{f_{\mathcal{S}\setminus{\mathcal{A}}^{*}(\mathcal{S})}(\mathcal{S}_2^{*})}{f(\mathcal{S}_2)}
\end{equation}
To prove \eqref{ineq2}, we still need to discuss two cases: $\mathcal{S}_2^{*} = \emptyset$ and $\mathcal{S}_2^{*} \neq \emptyset$. When $\mathcal{S}_2^{*} = \emptyset$, we have $f(\mathcal{S} \setminus {\mathcal{A}}^{*}(S))=f(\mathcal{S}_2)$, and \eqref{ineq2} holds. Next, we consider the case where $\mathcal{S}_2^{*} \neq \emptyset$ holds. The proof starts with one observation that
\begin{equation}\label{intermediate}
    f(\mathcal{S} \setminus {\mathcal{A}}^{*}(S)) \geq \max\{f(\mathcal{S} \setminus {\mathcal{A}}^{*}(S)), f(\mathcal{S}_1^{+})\},
\end{equation}
and then prove the following three inequalities:
\begin{equation}\label{intermediate-1}
    f(\mathcal{S} \setminus {\mathcal{A}}^{*}(S)) \geq (1-\xi)f(\mathcal{S}_2)
\end{equation}
\begin{equation}\label{intermediate-2}
    f(\mathcal{S}_{1}^{+}) \geq \xi\frac{1}{\alpha}f(\mathcal{S}_2)
\end{equation}
\begin{equation}\label{intermediate-3}
    \max\{(1-\xi, \xi \frac{1}{\alpha})\} \geq \frac{1}{\alpha+1}
\end{equation}
Next, if substitute \eqref{intermediate-1}, \eqref{intermediate-2}, and \eqref{intermediate-3} to \eqref{intermediate}, then \eqref{ineq2} is proved.

\begin{inparaenum}[1)]
\item \textit{Proof of inequalities $0 \leq \xi \leq 1$}: Since $f$ is non-negative and therefore by definition $\xi \geq 0$. For numerator of $\xi$, by submodularity, $f_{\mathcal{S}\setminus{\mathcal{A}}^{*}(\mathcal{S})}(\mathcal{S}_2^{*}) \leq f(\mathcal{S}_2^{*})$ and notice that $\mathcal{S}_2^{*}$ is a subset of  $\mathcal{S}_2$. Therefore,
\begin{align*}
        \xi & =  \frac{f_{\mathcal{S}\setminus{\mathcal{A}}^{*}(\mathcal{S})}(\mathcal{S}_2^{*})}{f(\mathcal{S}_2)} \\
        & \leq \frac{f(\mathcal{S}_2^{*})}{f(\mathcal{S}_2)} \leq 1
\end{align*}

\item \textit{proof of inequality \ref{intermediate-1}}: The proof can be done in two steps. Firstly, it can be verified using $f_{\mathcal{A}}(\mathcal{B})=f(\mathcal{A}\cup\mathcal{B})-f(\mathcal{A})$ that 
\begin{multline}
        f(\mathcal{S} \setminus {\mathcal{A}}^{*}(S))=f(\mathcal{S}_2)-
        f_{\mathcal{S} \setminus {\mathcal{A}}^{*}(S)}(\mathcal{S}_2^{*})\\
        +f_{\mathcal{S}_2}(\mathcal{S}_1)-f_{\mathcal{S}\setminus\mathcal{S}^{*}_{1}}(\mathcal{S}^{*}_{1})
\end{multline}
It should be noted that $f_{\mathcal{S}_2}(\mathcal{S}_1)-f_{\mathcal{S}\setminus\mathcal{S}^{*}_{1}}(\mathcal{S}^{*}_{1}) \geq 0$ for two following observations: 
\begin{inparaenum}[i)]
\item $f_{\mathcal{S}_2}(\mathcal{S}_1) \geq f_{\mathcal{S}_2}(\mathcal{S}_1^{*})$ since $f$ is monotone and $\mathcal{S}_1^{*} \subseteq \mathcal{S}_1$; 
\item $f_{\mathcal{S}_2}(\mathcal{S}_1^{*}) \geq f_{\mathcal{S}\setminus\mathcal{S}^{*}_{1}}(\mathcal{S}^{*}_{1})$ since $f$ is submodular and $\mathcal{S}_2 \subseteq \mathcal{S}\setminus\mathcal{S}^{*}_{1}$ (see also Fig. \ref{Venn diagram}). 
\end{inparaenum}
Then we have 
\begin{equation*}
    \begin{split}
                f(\mathcal{S} \setminus {\mathcal{A}}^{*}(S)) & \geq f(\mathcal{S}_2)-
        f_{\mathcal{S} \setminus {\mathcal{A}}^{*}(S)}(\mathcal{S}_2^{*})\\
        & = f(\mathcal{S}_2)-\xi f(\mathcal{S}_2)
    \end{split}
\end{equation*}
Inequality \ref{intermediate-1} proved.

\item \textit{Proof of inequality \ref{intermediate-2}}:
Since it is $\mathcal{S}_2^{*} \neq \emptyset$ which suggests that $\mathcal{S}_1^{+} \neq \emptyset$ and for all paths $a \in \mathcal{S}_1^{+}$ and all elements $b \in \mathcal{S}_2^{*}$ it is $f(a)\geq f(b)$, from Lemma \ref{a_geq_b}, we have 
\begin{equation*}
    \begin{split}
        f_{\mathcal{S}_1^{+}}(\mathcal{S}_2^{*}) & \leq |\mathcal{S}_2^{*}|f(\mathcal{S}_1^{+}) \\
        & \leq \alpha f(\mathcal{S}_1^{+})
    \end{split}
\end{equation*}
Since $|\mathcal{S}_2^{*}| \leq \alpha$. Overall,
\begin{equation}
    \begin{split}
        f(\mathcal{S}_1^{+}) & \geq \frac{1}{\alpha}f_{\mathcal{S}_1^{+}}(\mathcal{S}_2^{*}) \\
        & \geq \frac{1}{\alpha}f_{\mathcal{S}_1^{+} \cup \mathcal{S}_1^{+}}(\mathcal{S}_2^{*}) \\
        & = \frac{1}{\alpha}f_{\mathcal{S}\setminus{\mathcal{A}}^{*}(\mathcal{S})}(\mathcal{S}_2^{*}) \\
        & = \xi  \frac{1}{\alpha} f(\mathcal{S}_2)
    \end{split}
\end{equation}
where inequalities flow from top to down for submodularity, the definition of $\mathcal{S}_1^{+} \cup \mathcal{S}_1^{+}$, and the definition of $\xi$.

\item \textit{Proof of inequality \ref{intermediate-3}}:
Let $b=\frac{1}{\alpha}$. We complete the proof first for the case where $(1-\xi) \geq \xi b$, and then for the case where $(1-\xi) < \xi b$: when $(1-\xi) \geq \xi b$, $\max\{(1-\xi), \xi b\}=1-\xi$ and $\xi \leq \frac{1}{1+b}$; due to the latter, $1-\xi \geq \frac{b}{1+b}=\frac{1}{\alpha+1}$, which suggests inequality \ref{intermediate-3} holds; Finally, when $1-\xi < \xi b$, $\max\{(1-\xi), \xi b\}=\xi b$ and $\xi > \frac{1}{1+b}$; due to the latter, $\xi b > \frac{b}{1+b}=\frac{1}{1+\alpha}$, which also suggests inequality \ref{intermediate-3} holds. Thus the inequality \ref{ineq2} is proved.
\end{inparaenum}

\subsection{Proof of Inequality in Equation (\ref{s2_ineq3})}
We will prove the following,
\begin{equation}\label{ineq3}
    f(\mathcal{S} \setminus {\mathcal{A}}^{*}(S)) \geq \frac{1}{N-\alpha}f(\mathcal{S}_2)
\end{equation}
where  $\mathcal{S}=\mathcal{S}_1 \cup \mathcal{S}_2$ is the solution returned by our algorithm; ${\mathcal{A}}^{*}(S)$ is the optimal removal of $\mathcal{S}$; $\alpha=\max |{\mathcal{A}}|$ is the maximum number of removal from set $\mathcal{S}$ and $N$ is the number of robots.

The following two cases are enough to explain why \eqref{ineq3} holds.
\begin{itemize}
    \item when $\mathcal{S}_2^{*} = \emptyset$, i.e. all paths in $\mathcal{S}_1$ are removed and correspondingly no paths are removed in $\mathcal{S}_2$. Then $f(\mathcal{S} \setminus {\mathcal{A}}^{*}(S))=f(\mathcal{S}_2)$, and \eqref{ineq3} holds.
    \item when $\mathcal{S}_2^{*} \neq \emptyset$, that is there is at least one path left in $\mathcal{S}_1$ and choose one $s$ from any of them, then 
    \begin{equation}
        f(\mathcal{S} \setminus {\mathcal{A}}^{*}(S)) \geq f(s)
    \end{equation}
    since $f$ is non-decreasing. Moreover,
    \begin{equation}
        f(\mathcal{S}_2) \leq \sum_{v \in \mathcal{S}_2}f(v) \leq (N-\alpha)f(s) 
    \end{equation}
    where the first inequality holds due to submodularity and the second holds because the proposed algorithm will construct $\mathcal{S}_1$ and $\mathcal{S}_2$ such that $f(v^{\prime}) \geq f(v), \forall v^{\prime} \in \mathcal{S}_1,  v \in \mathcal{S}_2$.
\end{itemize}
Combine two cases, inequality \eqref{ineq3} is proved.

\subsection{Proof of Inequalities in Equation~\ref{Q_star2global_optimal}}
We will prove the following,
\begin{equation}\label{ineq4}
    f(\mathcal{Q}^{*}) \geq f^{*}
\end{equation}
where  $\mathcal{Q}^{*}$ is the optimal MOP solution to robots $i \in \mathcal{R} \setminus \mathcal{R}(\mathcal{S}_1)$, i.e. the optimal paths for robots corresponding to $\mathcal{S}_2$; $f^{*}$ is the optimal solution to RMOP.

We first prove the inequality \ref{ineq4}. Let $\Pi_{i}$ be the set of all feasible paths for robot $i$, which is hard to compute but is assumed here somehow known for analysis purposes. Then the ground set is defined as 
$$\Pi = \bigcup_{i=1}^{N}\Pi_{i} $$
By definition, $(\Pi, \mathcal{I})$ is a partition matroid if 
\begin{equation} \label{independent_set_large}
    \mathcal{I}=\{I \subseteq \Pi ~|~ |I \cap \Pi_{i}| \leq 1, \forall i=1,2,\ldots,N\}
\end{equation}
where independent set $\mathcal{I}$ represents all possible results of finding paths to $N$ robots.

Similarly, we have another partition matroid $(\Pi, \mathcal{I}^{\prime})$ with
\begin{equation}\label{independent_set_small}
    \begin{split}
        \mathcal{I}^{\prime}=\{I \subseteq \Pi ~|~ |I \cap \Pi_{i}| \leq a,~ & a=1, \forall i \in \mathcal{R}(\mathcal{S}_1); \\ & a=0, \forall i \in  \mathcal{R} \setminus \mathcal{R}(\mathcal{S}_1)\}
    \end{split}
\end{equation}
where independent set $\mathcal{I}^{\prime}$ represents all possible results of finding paths to robots in $\mathcal{R}(\mathcal{S}_1)$ and $\mathcal{I}^{\prime} \subseteq \mathcal{I}$.

In inequality \eqref{ineq4}, L.H.S, for any set $\mathcal{S}_1 \subseteq \Pi$ returned by algorithms with $|\mathcal{S}_1|=|\mathcal{R}(\mathcal{S}_1)|=\alpha$ such that $\mathcal{S}_1 \in \mathcal{I}$ and $\mathcal{S}_1 \in \mathcal{I}^{\prime}$. By definition of $\mathcal{Q}^{*}$
\begin{align}
    f(\mathcal{Q}^{*}) &=\max_{\mathcal{S}_2 \subseteq \Pi, \mathcal{S}_2 \cup \mathcal{S}_1 \in \mathcal{I}}f(\mathcal{S}_2)  \label{Q_definition} \\
    &=\max_{\mathcal{S}_2 \subseteq \Pi \setminus \mathcal{S}_1, \mathcal{S}_2 \cup \mathcal{S}_1 \in \mathcal{I}}f(\mathcal{S}_2) \label{Q_trivial}\\
    & \geq \min_{\Tilde{\mathcal{S}}_1 \subseteq \Pi, \Tilde{\mathcal{S}}_1 \in \mathcal{I}^{\prime}}~ \max_{\mathcal{S}_2 \subseteq \Pi \setminus \Tilde{\mathcal{S}}_1, \mathcal{S}_2 \cup \Tilde{\mathcal{S}}_1 \in \mathcal{I}}f(\mathcal{S}_2) \label{Q_min}\\
    & = \min_{\Tilde{\mathcal{S}}_1 \subseteq \Pi, \Tilde{\mathcal{S}}_1 \subseteq \in \mathcal{I}^{\prime}}~ \max_{\Tilde{\mathcal{S}} \subseteq \Pi, \Tilde{\mathcal{S}} \in  \mathcal{I}, \Tilde{\mathcal{S}}_1 \subseteq \Tilde{\mathcal{S}}}f(\Tilde{\mathcal{S}} \setminus \Tilde{\mathcal{S}}_1 ) \label{Q_change_variable}\\
    & = \min_{\Tilde{\mathcal{S}}_1 \subseteq \Pi, \Tilde{\mathcal{S}}_1 \subseteq \in \mathcal{I}^{\prime}}~ \max_{\Tilde{\mathcal{S}} \subseteq \Pi, \Tilde{\mathcal{S}} \in \mathcal{I}}f(\Tilde{\mathcal{S}} \setminus \Tilde{\mathcal{S}}_1 ) \label{Q_redundant}\\
    & \triangleq h
\end{align}
\eqref{Q_definition} is the definition of $\mathcal{Q}^{*}$; \eqref{Q_definition} to \eqref{Q_trivial} holds since we have a partition matroid with independent set defined as \eqref{independent_set_large} and the intersection of $\mathcal{S}_1$ and $\mathcal{S}_2$ will be empty; In \eqref{Q_trivial}, $\mathcal{S}_1$ is a specific subset of $\Pi$  and in the independent set $\mathcal{I}^{\prime}$. If we think $\mathcal{S}_1$ as an instantiation of a certain `set variable' $\Tilde{\mathcal{S}}_1$, we can find a minimal value (R.H.S of \eqref{Q_min}) through optimizing over $\Tilde{\mathcal{S}}_1$ and \eqref{Q_trivial} should be greater than minimal value, i.e. \eqref{Q_trivial} to \eqref{Q_min} holds; next we use the trick of changing of variables: let $\Tilde{\mathcal{S}}=\Tilde{\mathcal{S}}_1 \cup \mathcal{S}_2$ and $\mathcal{S}_2=\Tilde{\mathcal{S}} \setminus \Tilde{\mathcal{S}}_1$ due to the fact that $\Tilde{\mathcal{S}}_1$ and $\mathcal{S}_2$ are disjoint. As a result, \eqref{Q_min} to \eqref{Q_change_variable} holds; notice that in \eqref{Q_change_variable} the minimization operation over $\Tilde{\mathcal{S}}_1$ can guarantee that the solution satisfies $\Tilde{\mathcal{S}}_1 \subseteq \Tilde{\mathcal{S}}$ and we can remove the redundant constraint $\Tilde{\mathcal{S}}_1 \subseteq \Tilde{\mathcal{S}}$ in maximization, i.e. \eqref{Q_change_variable} to \eqref{Q_redundant} holds; and we define \eqref{Q_redundant} as $h$.

In the following, we basically show that min-max function is no less than max-min function. Notice that for any $\mathcal{S} \subseteq \Pi$ such that $\mathcal{S} \in \mathcal{I}$, and any set $\Tilde{\mathcal{S}}_1 \subseteq \Pi$ such that $\Tilde{\mathcal{S}}_1 \in \mathcal{I}^{\prime}$, it holds 
\begin{equation}
    \max_{\Tilde{\mathcal{S}} \subseteq \Pi, \Tilde{\mathcal{S}} \in \mathcal{I}}f(\Tilde{\mathcal{S}} \setminus \Tilde{\mathcal{S}}_1 ) \geq f(\mathcal{S \setminus \Tilde{\mathcal{S}}}_1)
\end{equation}
which implies:
\begin{equation}\label{any_S}
    \begin{split}
        h &\geq \min_{\Tilde{\mathcal{S}}_1 \subseteq \Pi, \Tilde{\mathcal{S}}_1  \in \mathcal{I}^{\prime}}~f(\mathcal{S \setminus \Tilde{\mathcal{S}}}_1) \\
        &=\min_{\Tilde{\mathcal{S}}_1 \subseteq \mathcal{S}, \Tilde{\mathcal{S}}_1  \in \mathcal{I}^{\prime}}~f(\mathcal{S \setminus \Tilde{\mathcal{S}}}_1) \\
    \end{split}
\end{equation}
Notice that \eqref{any_S} holds for all $\mathcal{S} \in \mathcal{I}$. As a result,

\begin{equation}
    h  \geq \max_{\Tilde{\mathcal{S}} \subseteq \Pi, \Tilde{\mathcal{S}} \in \mathcal{I}}~\min_{\Tilde{\mathcal{S}}_1 \subseteq \Tilde{\mathcal{S}}, \Tilde{\mathcal{S}}_1  \in \mathcal{I}^{\prime}}~f(\Tilde{\mathcal{S}} \setminus \Tilde{\mathcal{S}}_1) \label{max_S}
\end{equation}
    
Next we consider the minimization operation of \eqref{max_S}. For any $\Tilde{\mathcal{S}} \in \mathcal{I}$,
\begin{align}
    \min_{\Tilde{\mathcal{S}}_1 \subseteq \Tilde{\mathcal{S}}, \Tilde{\mathcal{S}}_1  \in \mathcal{I}^{\prime}}~f(\Tilde{\mathcal{S}} \setminus \Tilde{\mathcal{S}}_1)  \geq \min_{\Tilde{\mathcal{S}}_1 \subseteq \Tilde{\mathcal{S}}, |\Tilde{\mathcal{S}}_1| \leq \alpha}~f(\Tilde{\mathcal{S}} \setminus \Tilde{\mathcal{S}}_1) \label{min2alpha}
\end{align}
Reasons for \eqref{min2alpha} to hold: the L.H.S of \eqref{min2alpha} only allows remove $\Tilde{\mathcal{S}}_1 \in \mathcal{I}^{\prime}$ and $|\Tilde{\mathcal{S}}_1|$ can go up to $|\mathcal{R}(\mathcal{S}_1)|=\alpha$ (refer to definition of $\mathcal{I}^{\prime}$); by contrast, R.H.S of \eqref{min2alpha} is less constrained and can also remove up to $\alpha$ elements from $\Tilde{\mathcal{S}}$. Thus, the R.H.S can get the result no greater than L.H.S. As a result,
\begin{align}
        h & \geq \max_{\Tilde{\mathcal{S}} \subseteq \Pi, \Tilde{\mathcal{S}} \in \mathcal{I}}~\min_{\Tilde{\mathcal{S}}_1 \subseteq \Tilde{\mathcal{S}}, \Tilde{\mathcal{S}}_1  \in \mathcal{I}^{\prime}}~f(\Tilde{\mathcal{S}} \setminus \Tilde{\mathcal{S}}_1) \\
        &  \geq \max_{\Tilde{\mathcal{S}} \subseteq \Pi, \Tilde{\mathcal{S}} \in \mathcal{I}}~\min_{\Tilde{\mathcal{S}}_1 \subseteq \Tilde{\mathcal{S}}, |\Tilde{\mathcal{S}}_1| \leq \alpha}~f(\Tilde{\mathcal{S}} \setminus \Tilde{\mathcal{S}}_1) \label{back2alpha}\\
        & = f(\mathcal{S}^{*} \setminus \mathcal{A}^{*}(\mathcal{S}^{*})) \label{original_problem}\\
        & = f^{*}
\end{align}
In sum, 
$$f(\mathcal{Q}^{*}) \geq h \geq f^{*}.$$

\end{document}